\documentclass{article}
\PassOptionsToPackage{square,numbers,sort&compress}{natbib}
\usepackage[preprint]{neurips_2024}

\usepackage[utf8]{inputenc} 
\usepackage[T1]{fontenc}    
\usepackage[hidelinks]{hyperref}       
\usepackage{url}            
\usepackage{booktabs}       
\usepackage{amsfonts}       
\usepackage{nicefrac}       
\usepackage{microtype}      
\usepackage{comment}
\usepackage{amsthm}
\usepackage{mathtools}
\usepackage{amsmath}
\usepackage{physics}
\usepackage{array,makecell}
\usepackage{tabularx}
\usepackage{bbm}
\usepackage[toc,page]{appendix}
\usepackage{caption}
\usepackage{setspace}
\usepackage{xcolor}         
\usepackage{listings}
\usepackage{wrapfig}
\usepackage{algorithmicx}
\usepackage[ruled,vlined]{algorithm2e}
\mathtoolsset{showonlyrefs} 
\usepackage{xspace}

\newcommand{\indep}{\perp \!\!\! \perp}
\newcommand{\rev}[1]{{#1}}
\newcommand{\revt}[1]{{#1}}

\makeatletter
\newtheorem*{rep@theorem}{\rep@title}
\newcommand{\newreptheorem}[2]{%
\newenvironment{rep#1}[1]{%
 \def\rep@title{#2 \ref{##1}}%
 \begin{rep@theorem}}%
 {\end{rep@theorem}}}
\makeatother

\newtheorem{theorem}{Theorem}
\newtheorem{example}{Example}
\newreptheorem{theorem}{Theorem}
\newtheorem{lemma}{Lemma}
\newreptheorem{lemma}{Lemma}
\newtheorem{proposition}{Proposition}

\usepackage{xr}
\makeatletter
\newcommand*{\addFileDependency}[1]{
  \typeout{(#1)}
  \@addtofilelist{#1}
  \IfFileExists{#1}{}{\typeout{No file #1.}}
}
\makeatother
\newcommand{\ttmethod}{\texttt{PCP}\xspace}

\newcommand{\ttbaseline}{\texttt{Two-Staged}\xspace}
\newcommand{\ttnaive}{\texttt{Naive CP}\xspace}
\newcommand{\ttmethodjp}{\texttt{LOO-PCP}\xspace}
\newcommand{\ttwcp}{\texttt{WCP}\xspace}

\usepackage{color}

\title{Robust Conformal Prediction Using Privileged Information}

\author{%
   Shai Feldman \\
   Department of Computer Science \\
   Technion, Israel \\
   \texttt{shai.feldman@cs.technion.ac.il}
   \And
   Yaniv Romano \\
   Departments of Electrical and Computer Engineering and of Computer Science \\
   Technion, Israel \\
   \texttt{yromano@cs.technion.ac.il} \\
}
\begin{document}
\maketitle

\begin{abstract}
We develop a method to generate prediction sets with a guaranteed coverage rate that is robust to corruptions in the training data, such as missing or noisy variables. 
Our approach builds on conformal prediction, a powerful framework to construct prediction sets that are valid under the i.i.d assumption. Importantly, naively applying conformal prediction does not provide reliable predictions in this setting, due to the distribution shift induced by the corruptions. 
To account for the distribution shift, we assume access to privileged information (PI). The PI is formulated as additional features that explain the distribution shift, however, they are only available during training and absent at test time.
We approach this problem by introducing a novel generalization of weighted conformal prediction and support our method with theoretical coverage guarantees. 
Empirical experiments on both real and synthetic datasets indicate that our approach achieves a valid coverage rate and constructs more informative predictions compared to existing methods, which are not supported by theoretical guarantees.  




\end{abstract}

\section{Introduction}

\subsection{Motivation}
Uncertainty quantification plays a pivotal role in increasing the reliability of machine learning models. In this paper, we focus on situations where the training data is corrupted, e.g., due to missing variables or noisy labels. These corruptions are ubiquitous in high-stakes applications---such as diagnosing diseases, predicting financial outcomes, or personalizing treatment plans for patients---in which the data-collection process is complex, resource-intensive, or time-consuming~\cite{kahn1989statistical, lilienfeld1994foundations, piantadosi1997clinical, selvin2004statistical, fowler2013survey}. 

One way to enhance the trustworthiness of data-driven predictions is to provide an uncertainty set containing the correct outcome at a user-specified coverage rate, e.g., 90\%.
Conformal prediction (\texttt{CP})~\cite{vovk2005algorithmic} is a general framework for constructing such reliable prediction sets, however, it assumes that the training and test data samples are drawn i.i.d from the same distribution. This assumption does not hold in the problem setting we consider in this work, in which the training data is a corrupted or a biased version of the ground truth. For instance, consider a medical application in which training data have missing or incorrect labels for some patients in a non-random pattern. 
Another example is a situation where we have missing feature values in the training data but, at test-time, we observe the full set of features.
These examples illustrate common sources for a distribution shift between the training and test data, which breaks the coverage guarantee of traditional \texttt{CP} techniques.
In this work, we address this gap and propose a novel calibration technique, called \emph{privileged conformal prediction} (\ttmethod), which constructs provably valid uncertainty sets despite being employed with corrupted samples. Technically, we achieve this by utilizing privileged information---additional data available only during training time---to account for the distribution shift induced by the corruptions.

\subsection{Problem setup} 
Suppose we are given $n$ training samples $\{(X_i(M_i), Y_i(M_i), Z_i, M_i)\}_{i=1}^{n}$, where $X^\text{obs}_i = X_i(M_i)\in\mathcal{X}$ is the observed covariates, $Y^\text{obs}_i=Y_i(M_i)\in\mathcal{Y}$ is the observed response, $Z_i\in\mathcal{Z}$ is the privileged information (PI), and $M_i\in\{0,1\}$ is the corruption indicator. Specifically, if $M_i=1$ then either $X^\text{obs}_i$ or $Y^\text{obs}_i$ are corrupted, and if $M_i=0$, then $X^\text{obs}_i$ and $Y^\text{obs}_i$ correctly reflect the ground truth. \revt{In our setup, we require that the privileged information $Z_i$ explains the corruption occurences $M_i$. Formally, we assume that the clean data variables are independent of the corruption indicator given the privileged information, i.e., $(X(0), Y(0)) \indep M \mid Z$.}

At inference time, we aim to provide reliable predictions for the clean test response $Y^\text{test}=Y_{n+1}(0)$ given the clean version of the features: $X^\text{test}=X_{n+1}(0)$. That is, even though the observed $X^\text{obs}, Y^\text{obs}$ might be corrupted, the test $X^\text{test}, Y^\text{test}$ are always uncorrupted. Crucially, at test-time, we do not have access to the test privileged information $Z^\text{test}=Z_{n+1}$, nor the clean test label $Y^\text{test}$. Moreover, we assume that the PI $Z$ is always clean and correctly reflects the ground truth. We now emphasize the importance of this problem setup by providing several examples.
\begin{example}[Noisy response] Here, we refer to $Y_i(1)$ as a noisy version of the ground truth response $Y_i(0)$. For instance, $Y_i^\text{obs}=Y_i(M_i)$ could be a label obtained either by a non-expert annotator or an expert annotator, and $Z_i$ could be information about the annotator, such as their level of expertise. In this case, $M_i=0$ ($M_i=1$) indicates that the annotator chose the correct (wrong) label $Y_i^\textup{obs}=Y_{i}(0)$ ($Y_i^\textup{obs}=Y_{i}(1)$). In contrast, the features are always uncorrupted: $X_i(0)=X_i(1)$. Notice that the test $Y^\textup{test}=Y_{n+1}(0)$ always refers to the clean response. \revt{In this setup, since the PI, $Z_i$, is the information about the annotator, it is likely to explain the corruption appearances $M_i$. That is, it is sensible to believe that the assumption $ (X(0), Y(0)) \indep M \mid Z$ is approximately satisfied.}
\end{example}
\begin{example}[Missing features]
Here, $X_i(0)$ is the full clean feature vector, and $X_i(1) $ is a partial version of it, i.e., $X_{i,j}(1)=\text{`\texttt{NA}'}$ for some entries $j$. For example, consider an application where participants are requested to fill a user experience (UX) questionnaire, in which the goal is to predict user engagement. 
This trial consists of expert participants, who tend to fully answer the questionnaire, resulting in a full $X_i(0)$, and non-experts, who tend to partially answer it, resulting in the incomplete $X_i(1)$. The PI $Z_i$ could be the level of expertise of the participant.
Also, the response is always uncorrupted: $Y_i(0)=Y_i(1)$. We remark that at test time, the full feature vector $X^\textup{test}=X_{n+1}(0)$ is completely available. \revt{Since the PI $Z_i$ is the information about the participant, it is likely to explain the missing indices $M_i$. Therefore, for this choice of PI, we have a good reason to believe that the conditional independence requirement, $ X(0), Y(0) \indep M \mid Z$, is approximately satisfied.}
\end{example}
\begin{example}[Missing response]
\revt{Consider a medical setup where patients are being selected for a costly diagnosis, such as an MRI scan. Here, $X_i(0)=X_i(1)$ is the more standard medical measurements of the $i$-th patient, such as age, gender, medical history, and disease-specific measurements. The PI $Z_i$ is the information manually collected by the doctor to choose whether the patient should be examined by an MRI scan. This information is obtained through, e.g., a discussion of the doctor with the patient, or a physical examination, and could include, for instance, shortness of breath, swelling, blurred vision, etc. The response $Y_i(0)$ is the disease diagnosis obtained by the MRI scan, and $Y_i(1)=`\texttt{NA}'$. The missingness indicator $M_i$ equals 0 if the doctor decides to conduct an MRI scan, and 1 otherwise. At test time, our goal is to assist the doctors in future decisions before examining the patients, and hence the test PI $Z_\text{test}$ is unavailable. This task is relevant in situations where the number of available doctors is insufficient to examine all patients. Here, $Z_i$ explains the missingness $M_i$, and $M_i$ does not depend on $X_i$ or $Y_i$ given $Z_i$. 
}
\end{example}
With the above use cases in mind, 
our goal is to construct an uncertainty set $C(X^\textup{test})\subseteq \mathcal{Y}$ for the unknown clean test variable $Y^\textup{test}=Y_{n+1}(0)$ given the clean features $X^\textup{test}=X_{n+1}(0)$. Importantly, this uncertainty set should be statistically valid and satisfy the following coverage requirement:
\begin{equation}\label{eq:marginal_coverage}
\mathbb{P}(Y^\textup{test} \in C(X^\textup{test})) \geq 1-\alpha,
\end{equation}
where $1-\alpha \in (0,1)$ is a pre-specified coverage rate, e.g., 90\%.
This property is called \emph{marginal coverage}, as the probability is taken over all samples $\{(X_i(0),X_i(1),  Y_i(0), Y_i(1), Z_i,M_i)\}_{i=1}^{n+1}$, which are assumed to be drawn exchangeably (e.g., i.i.d.) from $P_{X(0), X(1), Y(0), Y(1), Z, M}$.
The challenge in achieving~\eqref{eq:marginal_coverage} 
lies in the fact that there is
a distribution shift between the training data $\{(X_i(M_i), Y_i(M_i))\}_{i=1}^{n}$ and the test data $(X_{n+1}(0), Y_{n+1}(0))$. Indeed, naively calibrating the model with the corrupted data may produce invalid uncertainty estimates~\cite{barber2022conformal}. Also, calibrating using only the clean data would result in biased predictions as the clean training samples are drawn from $P_{X(0), Y(0) \mid M=0}$, while the test samples are drawn from $P_{X(0), Y(0)}$. 

To bypass this bias, we assume that the privileged information explains away the corruption appearances. Formally, we require that the corruption indicator is independent of the clean data conditional on the value of the privileged information, $(X(0), Y(0)) \indep M \mid Z$. This assumption implies that our setting is a special case of covariate shift, with the covariates being the privileged information $Z$. Since the test PI $Z^\text{test}=Z_{n+1}$ is unknown at test time, conformal methods that account for covariate shift, such as \emph{weighted conformal prediction} (\texttt{WCP})~\cite{tibshirani2019conformal} cannot be applied directly in the setup.
In this paper, we re-formulate weighted conformal prediction and show how to construct uncertainty sets that satisfy the coverage requirement in~\eqref{eq:marginal_coverage} although $Z^\text{test}$ is unavailable.

\subsection{Our contribution}\label{sec:our_contribution}

We introduce \emph{privileged conformal prediction} (\ttmethod)---a novel calibration scheme that effectively handles corrupted data, and constructs provably valid uncertainty sets in the sense of~\eqref{eq:marginal_coverage}. Our key assumption is that the corruption indicator does not depend on the observed clean data given the privileged information, namely, $({Y}(0),{X}(0)) \indep M \mid Z$. This assumption implies that the privileged information explains away the corruption appearances. 
Building on \texttt{WCP}, we offer a specialized calibration scheme that carefully utilizes only the observed training privileged data $\{Z_i\}_{i=1}^n$ to attain a valid predictive inference at test time.
To enhance statistical efficiency, 
we further adapt \ttmethod for scarce data, building on leave-one-out arguments~\cite{barber2019jackknife, vovk2015cross}. 
Importantly, all methods we offer are supported by a theoretical valid coverage rate guarantee. Numerical experiments on both synthetic and real data show that naive conformal prediction techniques do not provide reliable uncertainty estimates, in contrast with the proposed \ttmethod. To the best of our knowledge, this work is the first to propose a calibration scheme that generates statistically valid prediction sets, assuming that the privileged information explains away the corruption appearances. 
\rev{Software implementing the proposed method and reproducing our experiments is available at \url{https://github.com/Shai128/pcp}.}

\section{Background and related work}

\subsection{Conformal prediction}\label{sec:cp}


Conformal Prediction (\texttt{CP})~\cite{vovk2005algorithmic} is a powerful framework for constructing prediction sets that hold a marginal coverage rate guarantee, in the sense of~\eqref{eq:marginal_coverage}. The general recipe to construct such prediction sets is as follows. First, split the data into a proper training set, indexed by $\mathcal{I}_1$, and a calibration set, indexed by $\mathcal{I}_2$. Then, fit a given learning model $\hat{f}$, e.g., a random forest or a neural network, on the training data. Next, evaluate the holdout prediction error of $\hat{f}$ by applying a non-conformity score function $\mathcal{S}(\cdot)\in\mathbb{R}$ to the calibration samples:
$S_i = \mathcal{S}(X_i,Y_i;\hat{f}), \forall i\in\mathcal{I}_2.$
Popular score functions include the absolute residual $\mathcal{S}(x,y;\hat{f})=|\hat{f}(x)-y|$ in regression cases, where $\hat{f}$ is a mean estimator, or $1-\hat{f}(x)_y$ in classification settings, where $\hat{f}(x)_y$ is the estimated probability of the $y$ label given $X=x$. The latter is known as the homogeneous prediction sets (HPS) score~\cite{vovk2005algorithmic}. Other score functions include the CQR score~\cite{romano2019conformalized} for regression tasks and the APS score~\cite{romano2020classification} for classification tasks.
Armed with the non-conformity scores, the conformal procedure proceeds by
computing the $(1+{1}/{|\mathcal{I}_2|})(1-\alpha)$-th empirical quantile of the calibration scores:
\begin{equation}\label{eq:cp_Q}
Q^\texttt{CP} = \left(1+ 1/{|\mathcal{I}_2|}\right)(1-\alpha)\text{-th empirical quantile of the scores } \{S_i\}_{i\in\mathcal{I}_2},
\end{equation}
where $1-\alpha$ is a user-specified coverage level. 
Lastly, the prediction set for the test point is given by
\begin{equation}
C^\texttt{CP}(X^\text{test}) = \{y: \mathcal{S}(X^\text{test},y;\hat{f}) \leq Q^\texttt{CP}\}.
\end{equation}
The above procedure is guaranteed to generate predictive sets with a valid marginal coverage~\eqref{eq:marginal_coverage} under the assumption that the calibration and test samples are exchangeable. We now turn to describe \emph{weighted conformal prediction} (\ttwcp) which is designed to handle exchangeability violations that arise from covariate shifts.

\subsection{Weighted conformal prediction}
Weighted Conformal Prediction (\texttt{WCP}) \cite{tibshirani2019conformal} extends the conformal prediction framework to handle covariate shifts. 
The key idea behind \texttt{WCP} is to weight the distribution of the calibration scores when taking their quantile in~\eqref{eq:cp_Q}, so that the weighted scores `look exchangeable' with the test non-conformity score. For the interest of space, we will not present the general form of \ttwcp, and instead focus on the setup presented in this work, in which $(X(0), Y(0)) \indep M \mid Z$. Under this assumption, the corruption indicator induces a covariate shift between the observed clean calibration samples and the test sample, which is explained by $Z$. That is, the clean calibration samples are drawn from $P_{X(0),Y(0)\mid M=0}$, while the test sample is drawn from $P_{X(0),Y(0)}$. Nevertheless, their distributions are equal conditionally on $Z$:
$P_{X(0),Y(0) \mid Z=z, M=0}=P_{X(0),Y(0) \mid Z=z}$.
With this in place, we follow the recipe of \ttwcp and construct a prediction set as follows. First, we compute the ratio of likelihoods between the test and train data: 
\begin{equation}\label{eq:pcp_weight}
w(z) 
=\frac{d P^\text{test}_{Z}(z)}{d P^\text{train}_{Z}(z)}=\frac{f^{\text{test}}_{Z}(z)}{f^{\text{test}}_{Z\mid  M=0}(z)}
=\frac{f^{\text{test}}_{Z}(z)}{f^{\text{test}}_{Z}(z) \frac{\mathbb{P}(M=0 \mid Z=z)}{\mathbb{P}(M=0)}}
= \frac{\mathbb{P}(M=0)}{\mathbb{P}(M=0 \mid Z=z)}.
\end{equation}
We define the set of uncorrupted calibration indexes as: $\mathcal{I}^{\text{uc}}_2 = \{j: \in \mathcal{I}_2, M_j=0\}$.
The normalized weights are formulated as:
\begin{equation}\label{eq:wcp_p}
p_i(Z^\text{test}) = \frac{w(Z_i)} {\sum_{k\in\mathcal{I}^\text{uc}_2} w(Z_k) + {w}(Z^\text{test})}, \ \ p_\text{test}(Z^\text{test}) =\frac{w(Z^\text{test})} {\sum_{k\in\mathcal{I}^\text{uc}_2} w(Z_k) + {w}(Z^\text{test})}
\end{equation}
Then, the calibration threshold for the test point is defined as the $1-\alpha$ empirical quantile of the weighted distribution of the scores:
\begin{equation}\label{eq:wcp_q}
Q^\texttt{WCP}(Z^\text{test}) := \text{Quantile}\left(1-\alpha; \sum_{i\in\mathcal{I}^\text{uc}_2} p_i(Z^\text{test}) \delta_{S_i} + p_\text{test}(Z^\text{test}) \delta_{\infty}\right),
\end{equation}
and, the prediction set for the test sample is defined similarly to \texttt{CP}:
\begin{equation}
C^\texttt{WCP}(X^\text{test}, Z^\text{test}) = \{y: \mathcal{S}(X^\text{test},y;\hat{f}) \leq Q^\texttt{WCP}(Z^\text{test})\}.
\end{equation}
Remarkably, \texttt{WCP} produces uncertainty sets that achieve the desired marginal coverage rate~\eqref{eq:marginal_coverage} despite the induced covariate shift. Nonetheless, to implement this method, we must have access to $Z^\text{test}$, which is required to obtain $w(Z^\text{test})$. In our problem setup, however, we assume that $Z^\text{test}$ is unavailable, and thereby \texttt{WCP} cannot be directly applied. This highlights the key challenge we aim to tackle in this paper, but before describing our method we first outline additional related work.

\subsection{Additional related work}

The concept of learning from privileged information was introduced by~\cite{vapnik2009new}, which proposes techniques to leverage additional knowledge available during training to improve the prediction accuracy and accelerate algorithm convergence rate. This idea has been further explored to train models that are more robust to distribution shifts in the context of domain adaptation~\cite{motiian2019domain,breitholtz2023towards, hoffman2016adaptive}. The method proposed in~\cite{xiao2023privileged} utilizes PI to handle datasets containing weak labels and to obtain more accurate predictions. Furthermore,~\cite{lopez2015unifying} combined model distillation with privileged information as a way to enhance learning from multiple models and data representations.
The integration of PI with traditional conformal prediction to generate more informative uncertainty estimates was explored in~\cite{yang2013learning, gauraha2018conformal}. This line of work stands in striking contrast with our proposal, as we present a novel robust conformal calibration procedure based on PI.  
More broadly, there have been developed conformal methods that advance beyond the exchangeability assumption, such as \ttwcp, among other contributions~\cite{barber2022conformal, prinster2022jaws, cauchois2022weak, cauchois2024robust, zaffran2023conformal, sesia2023adaptive, lee2024simultaneous}.
However, none of these works utilize PI to ensure the validity of the constructed prediction sets.

\section{Proposed method}
In this section, we present our main contribution, the \emph{privileged conformal prediction} (\ttmethod) method. Since the setup we study in this paper has not been explored in the literature of conformal prediction, we start by suggesting a naive approach to achieve~\eqref{eq:marginal_coverage}. Beyond serving as a baseline method for our \ttmethod, this naive approach also reveals the challenges involved in constructing valid prediction sets when the calibration data is corrupted. 
\subsection{A naive approach: Two-Staged Conformal}\label{sec:baseline_two_staged}

Recall that \ttwcp cannot be directly applied in our setup, since $Z^\text{test}$ is unknown. To overcome this, the naive approach presented below consists of two stages: (i) estimate the unknown $Z^\text{test}$ from the feature vector $X^\text{test}$, and (ii) employ \texttt{WCP} with the estimated privileged information. 

While this approach is intuitive, the estimation of $Z^\text{test}$ must be done in care: if the prediction of $Z$ is incorrect, then \ttwcp would not provide us the desired coverage guarantee. As a way out, instead of providing a point estimate, we will construct an interval $C^Z(X^\text{test})$ for $Z^\text{test}$ given $X^\text{test}$ using conformal prediction. This interval is guaranteed to contain the true PI $Z^\text{test}$ with probability $1-\beta$, where $\beta$ is a miscoverage rate of our choice, e.g., $\beta=0.01$.
Since we do not know which $z \in C^Z(X^\text{test})$ is the correct $Z^\text{test}$, we sweep over all possible elements $z\in C^Z(X^\text{test})$, compute their weights, $w(z)$, and take the largest weight:
\begin{equation}\label{eq:w_conservative}
    w^\text{conservative}(X^\text{test}) := \max_{z\in C^Z(X^\text{test})} w(z).
\end{equation}
The intuition behind taking the largest weight lies in Lemma~\ref{lem:non_decreasing_quantile}, which states that the larger the test weight is, the larger the threshold $Q^\ttwcp$ produced by \ttwcp, which, in turn, increases the size of the prediction set.
Armed with $w^\text{conservative}(X^\text{test})$, we can run \ttwcp with a nominal coverage level $1-\alpha+\beta$ using the clean calibration samples and their weights $\{(X_i^\text{obs}, Y_i^\text{obs}, w(Z_i))\}_{i\in \mathcal{I}^\text{uc}}$, and the conservative test weight, $w^\text{conservative}(X^\text{test})$. 
We denote the weighted score quantile provided by \ttwcp in~\eqref{eq:wcp_q} with this conservative test weight by $Q^\texttt{WCP}_\text{conservative}$.
The prediction set is therefore defined as:
\begin{equation}\label{eq:c_baseline}
    C^\ttbaseline(X^\text{test}) := \left\{y: \mathcal{S}(X^\text{test},y;\hat{f}) \leq Q^\texttt{WCP}_\text{conservative}\right\}.
\end{equation}
An outline of this procedure is given in Algorithm~\ref{alg:baseline} in Appendix~\ref{sec:alg_baseline}.
The proposition below states that the uncertainty set generated by this naive approach is guaranteed to contain the test label $Y^\text{test}$, despite the presence of corrupted labels in the calibration set. 
\begin{proposition}\label{prop:2staged_validity}
Suppose that $\{({X}_i(0),{X}_i(1), {Y}_i(0),{Y}_i(1),Z_i,  M_i)\}_{i=1}^{n+1}$ are exchangeable, the observed covariates are clean, i.e., $\forall i: X^\textup{obs}_i=X_i(0)={X}_i(1)$, the covariate shift assumption holds, i.e., $({X}(0), {Y}(0)) \indep M \mid Z$, and $P_{Z}$ is absolutely continuous with respect to $P_{Z \mid M=0}$. Then, the prediction set $C^\ttbaseline(X^\textup{test})$ from~\eqref{eq:c_baseline} achieves a valid coverage rate:
\begin{equation}
\mathbb{P}(Y^\textup{test} \in C^\ttbaseline(X^\textup{test})) \geq 1-\alpha.
\end{equation}
\end{proposition}
The proof is given in Appendix~\ref{sec:2staged_proof}.
While this two-staged approach constructs valid prediction sets, it has several limitations. First, it requires predicting not only $Y^\text{test}$ but also $Z^\text{test}$, and the prediction of the latter is anticipated to increase the uncertainty encapsulated in the resulting prediction set for $Y^\text{test}$. This algorithm also requires iterating over all $z \in C^Z(X^\text{test})$, which can be computationally expensive, especially when $Z$ is continuous or multi-dimensional. In addition, and perhaps more importantly, the prediction set $C^Z(X^\text{test})$ for $Z^\text{test}$ might contain unlikely, or off-support values of $Z$. This can lead to an extreme $w^\text{conservative}(X^\text{test})$, which, in turn, results in unnecessarily large prediction sets for $Y^\text{test}$.
Moreover, this naive method assumes that the calibration features $X_i^\text{obs}$ reflect the ground truth, i.e., $X_i^\text{obs}=X_i(0)$, and thus the coverage guarantee does not hold in situations where the features are missing or noisy. This discussion emphasizes the challenges in designing a calibration scheme that not only provides robust coverage guarantees but is also computationally and statistically efficient. In the next section, we present our main proposal which fully resolves all limitations of this naive approach.


\subsection{Our main proposal: Privileged Conformal Prediction}
In this section, we introduce our procedure to construct prediction sets with a valid coverage rate under the setting of corrupted samples.
We begin similarly to \texttt{CP}, as described in Section~\ref{sec:cp}, and split the data into a training set, $\mathcal{I}_1$, and a calibration set, $\mathcal{I}_2$. Next, we fit a predictive model $\hat{f}$ on the training data, and compute a non-conformity score for each calibration sample:
\begin{equation}
S_i = \mathcal{S}({X}^\text{obs}_i,{Y}^\text{obs}_i;\hat{f}), \forall i\in\mathcal{I}_2.
\end{equation}
Similarly to \ttwcp and \ttbaseline methods, we rely on the likelihood ratio of the training and test distributions, and compute the weight of the $i$-th sample:
$w_i := \frac{\mathbb{P}(M=0)}{\mathbb{P}(M=0 \mid Z=Z_i)}$.
The problem in \ttwcp is that the scores threshold $Q^\ttwcp(Z^\text{test})$ from~\eqref{eq:wcp_q} depends on $Z^\text{test}$. Here, we follow the intuition behind the two-staged baseline and propose an algorithm that provides a fixed threshold $Q^\ttmethod$ that is not a function of $Z^\text{test}$. This threshold can be thought of as a conservative estimate, or an upper bound of $Q^\ttwcp(Z^\text{test})$, which is based on the calibration data, and does not require $Z^\text{test}$.
To achieve this, we consider every calibration point ${i\in \mathcal{I}_2}$ as a test point, and run \texttt{WCP} as a subroutine to obtain the $i$-th score threshold $Q_i$. The final \ttmethod test score threshold, $Q^\ttmethod$, is defined as the $(1-\beta)$-th empirical quantile of the calibration thresholds $\{ Q_i\}_{i\in\mathcal{I}_2}$, where $\beta \revt{\in (0,\alpha)}$ is a level of our choice, e.g., $\beta=0.01$. 

Formally, we consider the $i$-th sample as a test point and compute the normalized weight of the $j$-th sample:
\begin{equation}
p^i_j = \frac{w_j} {\sum_{k\in\mathcal{I}_2^{\text{uc}}} w_k + {w}_{i}}, \ \forall i,j \in \mathcal{I}_2.
\end{equation}
Notice that $p^i_j$ extends the \ttwcp weights, $p_j$, from~\eqref{eq:wcp_p}, since $p_j = p_j^{n+1}$. Now, we compute the $i$-th threshold $Q_i$ by applying \ttwcp using the uncorrupted calibration data:
\begin{equation}\label{eq:Q_i}
Q_i := \text{Quantile}\left(1-\alpha + \beta; \sum_{j\in\mathcal{I}_2^{\text{uc}}} p^i_j \delta_{S_j} + p^i_{i} \delta_{\infty}\right),
\end{equation}
Next, we extract from $\{Q_i\}_{i \in \mathcal{I}_2}$ a conservative estimate of $Q^\ttwcp(Z^\text{test})=Q_{n+1}$, denoted by $Q^\ttmethod$:
\begin{equation}\label{eq:w2_Q}
Q^\ttmethod := \text{Quantile}\left(1-\beta; \sum_{i\in \mathcal{I}_2} \frac{1}{|\mathcal{I}_2|+1} \delta_{Q_i} + \frac{1}{|\mathcal{I}_2|+1} \delta_{\infty}\right).
\end{equation}
Finally, for a new input data $X^\text{test}$, we construct the prediction set for $Y^\text{test}$ as follows:
\begin{equation}\label{eq:weighted2_C}
C^\ttmethod(X^\text{test}) = \left\{y : \mathcal{S}(X^\text{test}, y, \hat{f}) \leq Q^\ttmethod \right\}.
\end{equation}
For convenience, Algorithm~\ref{alg:weighted2} summarizes the above procedure and Algorithm~\ref{alg:efficient_pcp} details a more efficient version of this procedure. We now show that the prediction sets constructed by~{\ttmethod} achieve a valid marginal coverage rate. 
\begin{theorem}\label{thm:validity}
Suppose that $\{({X}_i(0),{X}_i(1), {Y}_i(0),{Y}_i(1),Z_i,  M_i)\}_{i=1}^{n+1}$ are exchangeable, $(X(0),Y(0)) \indep M \mid Z$, and $P_{Z}$ is absolutely continuous with respect to $P_{Z \mid M=0}$. Then, the prediction set $C^\ttmethod(X^\textup{test})$ constructed according to Algorithm~\ref{alg:weighted2} achieves the desired coverage rate:
\begin{equation}
\mathbb{P}(Y^\textup{test} \in C^\ttmethod(X^\textup{test})) \geq 1-\alpha.
\end{equation}
\end{theorem}
The proof is given in Appendix~\ref{sec:validity_proof}. \revt{We remark that while Theorem~\ref{thm:validity} requires that the PI satisfies the conditional independence assumption, i.e., $(X(0),Y(0)) \indep M \mid Z$, in Appendix~\ref{sec:relaxing_cond_indep_assumption} we relax this assumption and provide a lower bound for the coverage rate for settings where the conditional independence assumption is not exactly satisfied.}
We pause here to emphasize the significance of Theorem~\ref{thm:validity}. 
The key challenge in proving this result lies in the fact that the $\{Q_i\}_{i\in \mathcal{I}_2 \cup \{n+1\}}$ are not exchangeable. 
This is attributed to the fact that for every $i\in \mathcal{I}_2^{\text{uc}}$, the threshold $Q_i$ is defined using its own score $S_i$, while the test $Q_{n+1}$ does not rely on its corresponding score $S_{n+1}=\mathcal{S}(X^\text{test}, Y^\text{test}, \hat{f})$. 
As a side comment, if the thresholds were exchangeable, the proof was much simpler, as $Q^\ttmethod$ would be greater than $Q_{n+1}$ with probability $1- \beta $. In this case, $C^\ttmethod$ includes $C^\ttwcp$ at a high probability, meaning that it achieves the desired coverage rate.
Due to the lack of exchangeability, the argument above is incorrect. Indeed, the validity of \ttmethod does not follow directly from the guarantee of \ttwcp, and it requires additional technical steps.

\revt{We now turn to discuss the role of $\beta$. First, we emphasize that Theorem~\ref{thm:validity} holds for any choice of $\beta\in (0,\alpha)$. Therefore, $\beta$ only affects the sizes of the uncertainty sets. Intuitively, as $\beta \rightarrow \alpha$, a higher quantile of the weighted distribution of the scores is taken, and a lower quantile of the $Q_i$’s is taken. Similarly, as $\beta  \rightarrow 0$ a lower quantile of the weighted distribution of the scores is taken, and a higher quantile of the $Q_i$’s is taken. An optimal $\beta$ can be considered as the $\beta$ that leads to the narrowest intervals. Such optimal $\beta$ can be practically computed with a grid of values for $\beta$ in $(0,\alpha)$, using a validation set.
Nonetheless, in our experiments, we directly chose $\beta$ that is close to 0.
In Appendix \ref{sec:beta_ablation} we conduct an ablation study analyzing the effect of $\beta$ on a synthetic dataset.
}

\revt{Lastly, we note that while the real ratios of likelihoods, $w_i$, are required to provide the validity guarantee in Theorem~\ref{thm:validity}, \ttmethod can be employed with estimates of $w_i$. These weights can be estimated in the following way. The first step is estimating the conditional corruption probability given $Z$, i.e., $\mathbb{P}(M=0 \mid Z=z)$, using the training and validation sets with any off-the-shelf classifier. We remark that this classifier can be fit on unlabeled data, as this classifier only requires the PI $Z$ and the corruption indicator $M$. We denote the model outputs by $\hat{p}(M=0 \mid Z=z)$. Next, we estimate the marginal corruption probability directly from the data: $\hat{p}(M=0)=\frac{1}{n}\sum_{i=1}^{n} M_i$. Finally, the estimated weights are computed according to~\eqref{eq:pcp_weight}:
$ \hat{w}_i = \hat{w}(z_i) = \frac{\hat{p}(M=0)}{\hat{p}(M=0 \mid Z=z_i)}. $
Even though \ttmethod is not guaranteed to attain the nominal coverage level when employed with the estimates $\hat{w}_i$, the experiments from Section~\ref{sec:cifar10_exp} indicate that it does achieve a conservative coverage rate in this case. The effect of inaccurate estimates of $w_i$ on the coverage rate attained by \ttmethod could be an exciting future direction to explore, perhaps by drawing on ideas from~\cite{lee2024simultaneous}.}



\begin{algorithm}[h]
	 \caption{Privileged Conformal Prediction (\ttmethod)}
	\label{alg:weighted2}
	
	\textbf{Input:}
	\begin{algorithmic}
		\State Data $({X}^\text{obs}_i, {Y}^\text{obs}_i, Z_i, M_i) \in \mathcal{X} \cross \mathcal{Y} \cross \mathcal{Z}\cross\{0,1\}, 1\leq i \leq n$, weights $\{w_i\}_{i=1}^{n}$, miscoverage level $\alpha \in (0,1)$, level $\beta \revt{\in (0,\alpha)}$, an algorithm $\hat{f}$, a score function $\mathcal{S}$, and a test point $X^\text{test}=x$.
	\end{algorithmic}
	
	\textbf{Process:}
	\begin{algorithmic}
            \State Randomly split $\{1,...,n\}$ into two disjoint sets $\mathcal{I}_1, \mathcal{I}_2$.
            \State Fit the base algorithm $\hat{f}$ on the training data $\{({X}^\text{obs}_i, {Y}^\text{obs}_i)\}_{i\in\mathcal{I}_1}$.
            \State Compute the scores $S_i=\mathcal{S}({X}^\text{obs}_i, {Y}^\text{obs}_i;\hat{f})$ for the calibration samples, $i\in \mathcal{I}_2^\text{uc}$.  
            \State Compute a threshold $Q_i$ for each calibration sample according to~\eqref{eq:Q_i}.
            \State Compute $Q^\ttmethod$, the $(1- \beta)$ quantile of $\{Q_i\}_{i\in \mathcal{I}_2}$, according to~\eqref{eq:w2_Q}.
	\end{algorithmic}
	
	\textbf{Output:}
	\begin{algorithmic}
		\State Prediction set $C^\ttmethod(x)=\{y: \mathcal{S}(x,y;\hat{f}) \leq Q^\ttmethod\}$.
	\end{algorithmic}
\end{algorithm}

\subsection{Privileged Conformal Prediction for scarce data}\label{sec:pcp_scarce}
In this section, we present an adaptation of \ttmethod to handle situations where the sample size is small.
While \ttmethod is computationally light, it requires splitting the data into training and calibration sets. This restriction is significant for small datasets in which the reduction in computations from the data splitting comes at the expense of statistical efficiency. 
To avoid data splitting, we build on the leave-one-out jackknife+ method~\cite{barber2019jackknife, gupta2020nested}, and, in particular, its weighted version \texttt{JAW}~\cite{prinster2022jaws}. The method we propose, which we refer to as \ttmethodjp, better utilizes the training data compared to \ttmethod. 
For the interest of space, we refer to Appendix~\ref{sec:pcp_scarce_algorithm} for the description of \ttmethodjp.
The following Theorem states that prediction set $C^\ttmethodjp$ constructed by \ttmethodjp is guaranteed to achieve a valid coverage rate under our setup. In Appendix~\ref{sec:validity_pcp_loo} we provide the proof, which relies on results from~\cite{prinster2022jaws}.
\begin{theorem}\label{thm:validity_j}
Suppose that $\{({X}_i(0),{X}_i(1), {Y}_i(0),{Y}_i(1),Z_i,  M_i)\}_{i=1}^{n+1}$ are exchangeable, $(X(0),Y(0)) \indep M \mid Z$, and $P_{Z}$ is absolutely continuous with respect to $P_{Z \mid M=0}$. Then, the prediction set $C^\ttmethodjp(X^\textup{test})$ constructed according to Algorithm~\ref{alg:weighted2_j} satisfies:
\begin{equation}
\mathbb{P}(Y^\textup{test} \in C^\ttmethodjp(X^\textup{test})) \geq 1-2\alpha.
\end{equation}
\end{theorem}
We note that in contrast to split \texttt{CP}, here, the coverage guarantee appears with a factor $2$ in $\alpha$. We refer the reader to~\cite{barber2019jackknife} for a detailed explanation of why the factor $2$ is necessary and cannot be removed. Nonetheless, it is well-known that this jackknife approach greatly improves statistical efficiency compared to split conformal methods.

\section{Applications}\label{sec:experiments}
In this section, we exemplify our proposal in three real-life applications. In all experiments, we randomly split the data into training, validation, calibration, and test sets. We fit a base learning model on the training data and use the validation set to avoid overfitting. 
We calibrate the model using the calibration data with the proposed \ttmethod or with a baseline technique, and evaluate the performance on the test set. 
In all experiments, the calibration schemes are applied to achieve a $1-\alpha=90\%$ marginal coverage rate. We use the \texttt{CQR}~\cite{romano2019conformalized} non-conformity score in regression tasks, and the \emph{homogeneous prediction sets} (\texttt{HPS}) non-conformity score~\cite{vovk2005algorithmic, lei2013distribution} in classification tasks.
Appendix~\ref{sec:experimental_setup} describes the full details about the network architecture, training strategy, datasets, corruption technique, and this experimental protocol.
The specific formulation of the PI is described in each experiment.
In this section, we focus on three use cases: causal inference, missing response, and noisy response. We demonstrate the applicability of \ttmethod on more datasets, and under different corruptions, including additional causal inference tasks in Appendix~\ref{sec:additional_causal_inference_exp}, more response corruptions in Appendix~\ref{sec:cifar10c_exp} and in Appendix~\ref{sec:tabular_noisy_response_exp}, and missing features settings in Appendix~\ref{sec:missing_features_exp}.



\subsection{Causal inference: semi-synthetic example}\label{sec:causal_inference_exp}
We begin with a causal inference example, in which the goal is to obtain inference for individual treatment effects~\cite{hernan2010causal}.
In this setting, ${X}_i(0)={X}_i(1)\in \mathcal{X}$ denotes the features, $Z_i$ denotes the privileged information, $M_i \in \{0,1\}$ denotes the binary treatment indicator, and $Y_i(0), Y_i(1) \in \mathbb{R}$ denote the counterfactual outcomes under control and treatment conditions, respectively. Recall that we only observe $Y_i^\text{obs}=Y_i(M_i)$ and that the PI explains the treatment pattern $(X(0), Y(0)) \indep M \mid Z$. 
In this experimental setup, our goal is to construct a prediction set that covers the true test potential outcome under control conditions, i.e., $Y_{n+1}(0)$, at a user-specified level $1-\alpha=90\%$. Alternatively, we could also aim to predict the outcome under treatment $Y_{n+1}(1)$. However, in this experiment, we focus on $Y_{n+1}(0)$ since the dataset we use is highly imbalanced and there are few samples from the treatment group. This task is compelling since it can be used to generate a valid uncertainty interval for the individual treatment effect (ITE), $Y_i(1) - Y_i(0)$, which is a great interest for many problems~\cite{brand2010benefits, morgan2001counterfactuals, xie2012estimating, florens2008identification}. For instance,  the work in~\citep{lei2021conformal} shows how to construct a valid interval for the ITE by combining intervals for $Y_{n+1}(0)$ and $Y_{n+1}(1)$. We remark that providing statistically valid prediction intervals for $Y_{n+1}(0)$ is challenging due to the distribution shift between the observed control responses, which are drawn from $P_{Y(0) \mid M=0}$, whereas the test control response is drawn from $P_{Y(0)}$. Moreover, in this example, we intentionally design $M_i$ to induce such a distribution shift; see Appendix~\ref{sec:general_exp_setup} for more details on the definition of $M_i$.

We test the applicability of our method on the semi-synthetic Infant Health and Development Program (IHDP) dataset~\cite{hill2011bayesian}, in which the objective is to find the effect of specialist home visits on a child's future cognitive test scores. That is, the feature vector $X_i$ contains covariates describing the child's characteristics, the treatment $M_i$ is the specialist home visits indicator, and the potential outcomes $Y_i(0), Y_i(1)$ are the future cognitive test scores.
Since this dataset does not originally contain a privileged information variable, we artificially define it as the entry in $X_i$ that correlates the most with $Y_i(0)$.
This feature is then removed from $X_i$, so it is unavailable at inference time.

Since the IHDP dataset contains only~747 samples, we apply \ttmethodjp in this example and compare it to the following calibration techniques. The first method is a naive \texttt{jackkife+}, which uses only the control samples 
and does not account for distribution shifts. The second and third techniques are two versions of \texttt{JAW}~\cite{prinster2022jaws}, which is a weighted conformal version of the jackknife+. The first version (\texttt{Naive WCP}) naively uses an estimate of $P_{M\mid X}$ as the likelihood ratio weights instead of $P_{M\mid Z}$, as $Z^\text{test}$ is unknown.
The second (\texttt{Infeasible WCP}) is an \textbf{infeasible} method which requires access to the unknown test privileged information $Z^\text{test}$ for the computation of the likelihood ratio weights, which can be considered as an oracle calibration process. 
Importantly, the infeasible \texttt{JAW} and the proposed method use the true corruption probabilities when computing the weights $w(z)$ from~\eqref{eq:pcp_weight}.

Figure~\ref{fig:ihdp} reports the coverage rates and interval lengths of each calibration scheme. This figure shows that naive \texttt{jackkife+} and the naive \texttt{JAW} achieve a lower coverage rate than desired. This is anticipated, as both schemes do not accurately account for the distribution shift.
In contrast, the infeasible \texttt{JAW} and \ttmethod achieve the desired coverage rate. This is not a surprise, as \texttt{JAW} is guaranteed to attain the nominal coverage level when applied with the correct weights~\cite[Theorem 1]{prinster2022jaws}. However, this method is infeasible to implement in contrast with our proposal, which, according to Theorem~\ref{thm:validity_j}, is guaranteed to cover the response at the desired rate without using the test privileged information. Furthermore, Figure~\ref{fig:ihdp} reveals that \ttmethod constructs intervals with approximately the same width as the ones generated by the infeasible \texttt{JAW}. This indicates that we do not lose much in terms of statistical efficiency by not having access to the test privileged information $Z^\text{test}$.

\begin{figure}[h]
  \centering
         \includegraphics[width=0.9\textwidth]{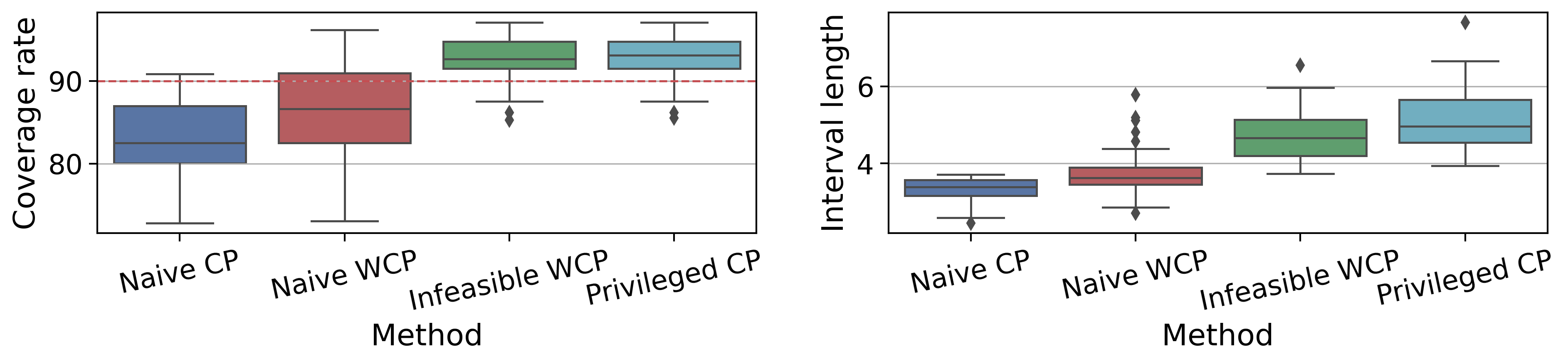}
     \caption{\textbf{Causal inference experiment: IHDP dataset.} The coverage rate and average interval length achieved by naive jackknife+ (\texttt{Naive CP}), naive \texttt{JAW} which considers only $X$ to cope with the distribution shift (\texttt{Naive WCP}), an infeasible \texttt{JAW} which uses $Z^\text{test}$ (\texttt{Infeasible WCP}), and the proposed method (\texttt{Privileged CP}).
    The metrics are evaluated over 50 random data splits.}
\label{fig:ihdp}%
\end{figure}%

\subsection{Missing response variable: semi-synthetic example}\label{sec:missing_response_exp}
In this section, we study the performance of \ttmethod and compare it to baselines in a missing response setting using six real datasets: Facebook1,2~\cite{facebook_data}, Bio~\cite{bio_data}, House~\cite{house_data}, Meps19~\cite{meps19_data} and Blog~\cite{blog_data}. Since these datasets do not originally contain privileged information, we artificially define $Z_i$ as the feature from $X_i$ that correlates the most with $Y_i$ and then remove it from $X_i$. Furthermore, since all response variables are present in these datasets, we artificially remove the responses in 20\% of the samples. We intentionally set the missing probability in a way that induces a distribution shift between the missing and observed variables. In Appendix~\ref{sec:general_exp_setup} we provide the full details about the corruption process and how we impute the missing data.

We compare the proposed method (\ttmethod) to the following calibration schemes: a naive conformal prediction (\ttnaive); a naive \texttt{WCP}, which considers only $X$ to cope with the distribution shift; the two-staged baseline (\ttbaseline); and an infeasible weighted conformal prediction (\texttt{Infeasible WCP}) which has access to the test privileged information $Z^\text{test}$. 
Importantly, the baseline \ttbaseline, the infeasible \texttt{WCP}, and \ttmethod use the real corruption probabilities when computing the weights $w(z)$ in~\eqref{eq:pcp_weight}. In contrast, \texttt{Naive WCP} estimates the corruption probability conditioned on $X$ from the data.  
Figure~\ref{fig:missing_y} presents the performance of each calibration scheme, showing that the naive approach (\ttnaive) consistently produces invalid prediction intervals. This is anticipated, as \ttnaive does not provide guarantees under distribution shifts. Figure~\ref{fig:missing_y} also shows that \ttbaseline generates too wide intervals, resulting in a conservative coverage rate of approximately 95\%. By contrast, the infeasible \texttt{WCP} and the proposed \ttmethod consistently achieve the desired 90\% level. Crucially, \ttmethod is comparable in the interval length to the infeasible \texttt{WCP}. In conclusion, this experiment demonstrates that \ttmethod constructs intervals that are both reliable and informative.

\begin{figure}[ht]
        \includegraphics[width=0.98\textwidth,trim={0 0 0 2.5cm},clip]{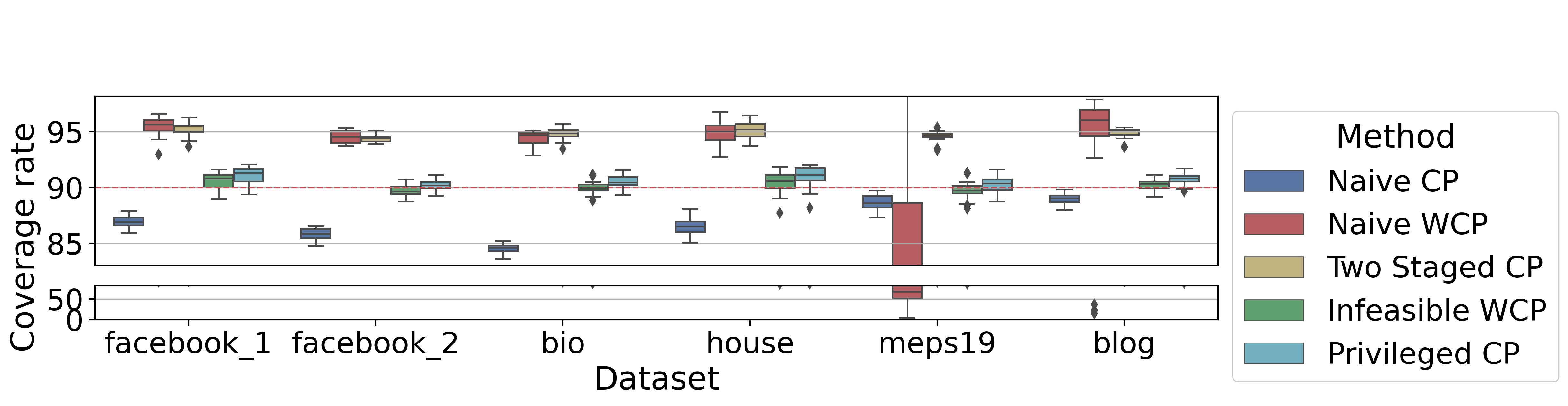}\\
        
        \includegraphics[width=0.77\textwidth]{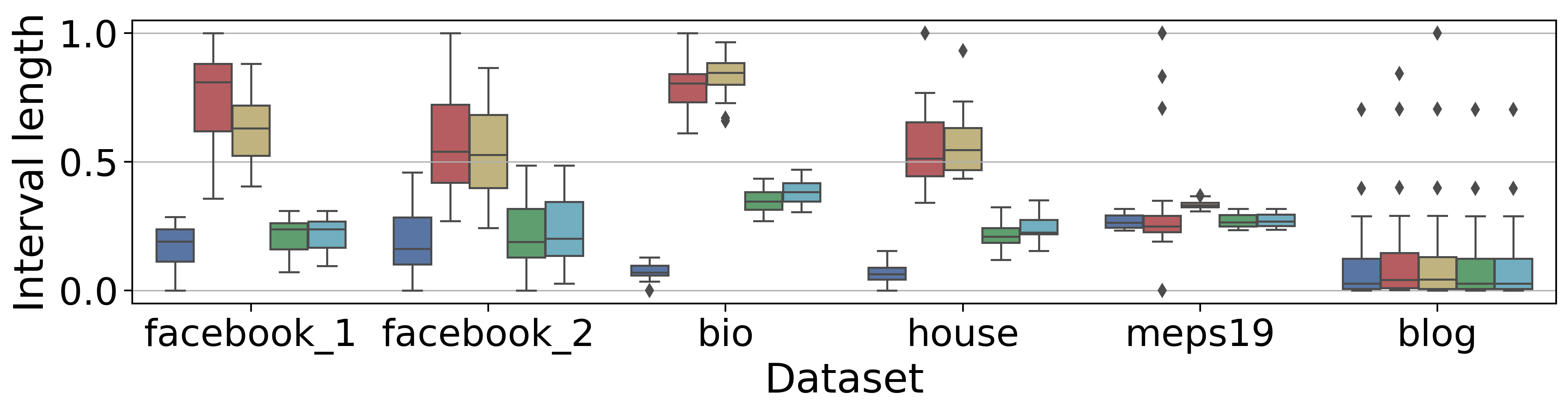}
     \caption{\textbf{Missing response experiment.} The coverage rate and average interval length obtained by 
     various methods; see text for details.
    Performance metrics are evaluated over 20 random data splits.}
\label{fig:missing_y}%
\end{figure}%

\subsection{Noisy response variable: real example}\label{sec:cifar10_exp}
In what follows, we examine the performance of the proposed technique on the CIFAR-10N~\cite{wei2022learning} image recognition dataset that contains noisy labels. Here, $X$ is an image of one out of ten possible objects, and $Y$ is its corresponding label. The noisy response, ${Y}(1)$, is the label annotated by one human annotator, while $Y(0)$ denotes the clean label obtained from CIFAR-10~\cite{krizhevsky2009learning}. That is, $M=0$ indicates that the annotator correctly labeled the image. 
Similarly to~\cite{xiao2023privileged}, we define the privileged data as information about the annotators. Specifically, the variable $Z_i$ contains two features: (i) the number of unique labels suggested by three annotators for the $i$-th sample, and (ii) the time took to annotate the corresponding sample batch, which contains ten images. 
In this experiment, we compare our method (\ttmethod) to the following calibration schemes: a naive conformal prediction, applied either with the noisy labels \texttt{Naive CP (clean + noisy)} or ignoring them \texttt{Naive CP (only clean)}; the two-staged baseline (\texttt{Two Staged CP}); an infeasible \ttwcp (\texttt{Infeasible WCP}) which assumes access to the unknown test privileged information  $Z^\text{test}$.
Additionally, since the corruption probabilities are not given in this dataset, we estimate them from the data and use these estimates to compute the weights $w$ in~\eqref{eq:pcp_weight}. 

Figure~\ref{fig:cifar10} presents each calibration scheme's coverage rate and uncertainty set size. 
This figure shows that \ttnaive applied with noisy labels tends to overcover the clean label. This behavior is consistent with the work in~\cite{label_noise, sesia2023adaptive}, which suggests that naive \texttt{CP} constructs conservative uncertainty sets when employed on data with dispersive label-noise. Figure~\ref{fig:cifar10} also indicates that calibrating the model only on the clean samples leads to invalid prediction sets that tend to undercover the clean test label. Observe also that the two-stage baseline is overly conservative, as it encapsulates the error in predicting both $Z^\text{test}$ and the label. In contrast, the coverage rate of infeasible \texttt{WCP} and our proposed \ttmethod is much closer to the desired level, yet slightly conservative. We suggest two possible explanations for this behavior: (i) the weights used are only estimates of the true likelihood ratios; (ii) in this real data we consider here, the PI may not fully explain the corruption mechanism.
This highlights the robustness of our method to violations of our assumptions in the specific use-case studied here.
Lastly, we remark that the prediction sets of \ttmethod have a similar set size to the sets constructed by the infeasible \texttt{WCP}, which is in line with the results from Section~\ref{sec:causal_inference_exp} and Section~\ref{sec:missing_response_exp}. 

\begin{figure}[h]
  \centering
       \includegraphics[width=0.9\textwidth,trim={0 0 0 0.25cm},clip]{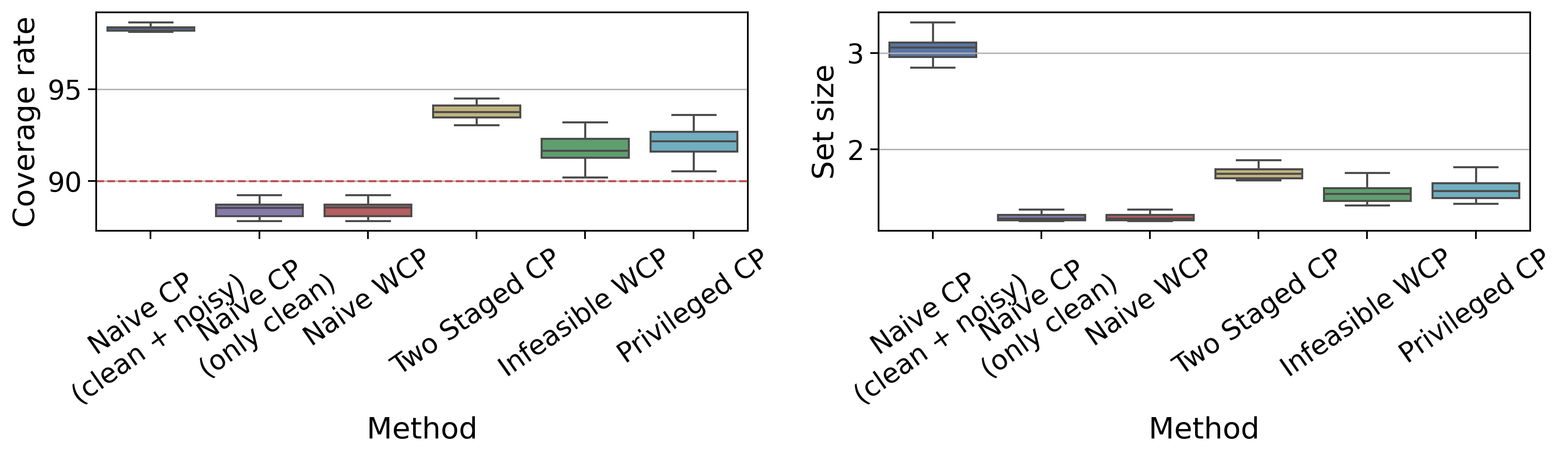}
        
     \caption{\textbf{Noisy response experiment: CIFAR-10N dataset.} Average coverage and set size obtained by 
     various methods; see text for details.
    The metrics are evaluated over 20 random data splits.}
\label{fig:cifar10}%
\end{figure}%


\section{Discussion and impact statement}\label{sec:conclusion}
In this paper, we introduced \ttmethod, a novel calibration scheme to reliability quantify prediction uncertainty in situations where the training data is corrupted. 
The validity of our proposal is supported by theoretical guarantees and demonstrated in numerical experiments. 
The key assumption behind our method is that the features and responses are independent of the corruption indicator given the privileged information.  
This conditional independence resembles the strong ignorability assumption in causal inference~\cite{rubin1978bayesian, rosenbaum1983central, imbens2015causal}.
While acquiring PI that satisfies this requirement can be challenging, our work relaxes the strong ignorability assumption, as the confounders are allowed to be absent during inference time.
An additional restriction we make is that the true conditional corruption probability 
must be known to provide a theoretical coverage validity. 
However, our numerical experiments indicate that estimating these probabilities leads to reliable uncertainty estimates. A promising future direction would be to theoretically analyze the effect of inaccurate weights on the coverage guarantee, e.g., by borrowing ideas from~\cite{lee2024simultaneous}.
Finally, we should note that there are potential social implications of our method, akin to many other works that aim to advance the ML field. 

\begin{ack}
\rev{Y.R. and S.F. were supported by the ISRAEL SCIENCE FOUNDATION (grant No. 729/21). Y.R. thanks the Career Advancement Fellowship, Technion. Y.R. and S.F. thank Stephen Bates for insightful discussions and for providing feedback on this manuscript.}
\end{ack}

\bibliographystyle{unsrt}
\bibliography{bibliography}

\begin{thebibliography}{10}

\bibitem{kahn1989statistical}
H.A. Kahn and C.T. Sempos.
\newblock {\em Statistical Methods in Epidemiology}.
\newblock Monographs in epidemiology and biostatistics. Oxford University Press, 1989.

\bibitem{lilienfeld1994foundations}
David~E Lilienfeld and Paul~D Stolley.
\newblock {\em Foundations of epidemiology}.
\newblock Oxford University Press, USA, 1994.

\bibitem{piantadosi1997clinical}
Steven Piantadosi.
\newblock {\em Clinical trials: a methodologic perspective}.
\newblock John Wiley \& Sons, 1997.

\bibitem{selvin2004statistical}
Steve Selvin.
\newblock {\em Statistical analysis of epidemiologic data}, volume~35.
\newblock Oxford University Press, 2004.

\bibitem{fowler2013survey}
Floyd~J Fowler~Jr.
\newblock {\em Survey research methods}.
\newblock Sage publications, 2013.

\bibitem{vovk2005algorithmic}
Vladimir Vovk, Alex Gammerman, and Glenn Shafer.
\newblock {\em {Algorithmic Learning in a Random World}}.
\newblock Springer, New York, NY, USA, 2005.

\bibitem{barber2022conformal}
Rina~Foygel Barber, Emmanuel~J. Cand{\`e}s, Aaditya Ramdas, and Ryan~J. Tibshirani.
\newblock {Conformal prediction beyond exchangeability}.
\newblock {\em The Annals of Statistics}, 51(2):816 -- 845, 2023.

\bibitem{tibshirani2019conformal}
Ryan~J Tibshirani, Rina Foygel~Barber, Emmanuel Candes, and Aaditya Ramdas.
\newblock Conformal prediction under covariate shift.
\newblock {\em Advances in neural information processing systems}, 32, 2019.

\bibitem{barber2019jackknife}
Rina~Foygel Barber, Emmanuel~J Cand{\`e}s, Aaditya Ramdas, and Ryan~J Tibshirani.
\newblock Predictive inference with the jackknife+.
\newblock {\em The Annals of Statistics}, 49(1):486 -- 507, 2021.

\bibitem{vovk2015cross}
Vladimir Vovk.
\newblock Cross-conformal predictors.
\newblock {\em Annals of Mathematics and Artificial Intelligence}, 74:9--28, 2015.

\bibitem{romano2019conformalized}
Yaniv Romano, Evan Patterson, and Emmanuel Candes.
\newblock Conformalized quantile regression.
\newblock {\em Advances in neural information processing systems}, 32, 2019.

\bibitem{romano2020classification}
Yaniv Romano, Matteo Sesia, and Emmanuel Cand{\`e}s.
\newblock Classification with valid and adaptive coverage.
\newblock In {\em Advances in Neural Information Processing Systems}, volume~33, pages 3581--3591, 2020.

\bibitem{vapnik2009new}
Vladimir Vapnik and Akshay Vashist.
\newblock A new learning paradigm: Learning using privileged information.
\newblock {\em Neural networks}, 22(5-6):544--557, 2009.

\bibitem{motiian2019domain}
Saeid Motiian.
\newblock {\em Domain Adaptation and Privileged Information for Visual Recognition}.
\newblock West Virginia University, 2019.

\bibitem{breitholtz2023towards}
Adam Breitholtz.
\newblock {\em Towards practical and provable domain adaptation}.
\newblock PhD thesis, Chalmers Tekniska Hogskola (Sweden), 2023.

\bibitem{hoffman2016adaptive}
Judith Hoffman.
\newblock {\em Adaptive learning algorithms for transferable visual recognition}.
\newblock University of California, Berkeley, 2016.

\bibitem{xiao2023privileged}
Yanshan Xiao, Zexin Ye, Liang Zhao, Xiangjun Kong, Bo~Liu, Kemal Polat, and Adi Alhudhaif.
\newblock Privileged information learning with weak labels.
\newblock {\em Applied Soft Computing}, 142:110298, 2023.

\bibitem{lopez2015unifying}
David Lopez-Paz, L{\'e}on Bottou, Bernhard Sch{\"o}lkopf, and Vladimir Vapnik.
\newblock Unifying distillation and privileged information.
\newblock {\em International Conference on Learning Representations}, 2016.

\bibitem{yang2013learning}
Meng Yang, Ilia Nouretdinov, and Zhiyuan Luo.
\newblock Learning by conformal predictors with additional information.
\newblock In {\em Artificial Intelligence Applications and Innovations: 9th IFIP WG 12.5 International Conference, AIAI 2013, Paphos, Cyprus, September 30--October 2, 2013, Proceedings 9}, pages 394--400. Springer, 2013.

\bibitem{gauraha2018conformal}
Niharika Gauraha, Lars Carlsson, and Ola Spjuth.
\newblock Conformal prediction in learning under privileged information paradigm with applications in drug discovery.
\newblock In {\em Conformal and Probabilistic Prediction and Applications}, pages 147--156. PMLR, 2018.

\bibitem{prinster2022jaws}
Drew Prinster, Anqi Liu, and Suchi Saria.
\newblock Jaws: Auditing predictive uncertainty under covariate shift.
\newblock {\em Advances in Neural Information Processing Systems}, 35:35907--35920, 2022.

\bibitem{cauchois2022weak}
Maxime Cauchois, Suyash Gupta, Alnur Ali, and John Duchi.
\newblock Predictive inference with weak supervision.
\newblock {\em arXiv preprint arXiv:2201.08315}, 2022.

\bibitem{cauchois2024robust}
Maxime Cauchois, Suyash Gupta, Alnur Ali, and John~C Duchi.
\newblock Robust validation: Confident predictions even when distributions shift.
\newblock {\em Journal of the American Statistical Association}, pages 1--66, 2024.

\bibitem{zaffran2023conformal}
Margaux Zaffran, Aymeric Dieuleveut, Julie Josse, and Yaniv Romano.
\newblock Conformal prediction with missing values.
\newblock In {\em International Conference on Machine Learning}, pages 40578--40604. PMLR, 2023.

\bibitem{sesia2023adaptive}
Matteo Sesia, YX~Wang, and Xin Tong.
\newblock Adaptive conformal classification with noisy labels.
\newblock {\em arXiv preprint arXiv:2309.05092}, 2023.

\bibitem{lee2024simultaneous}
Yonghoon Lee, Edgar Dobriban, and Eric~Tchetgen Tchetgen.
\newblock Simultaneous conformal prediction of missing outcomes with propensity score $\epsilon$-discretization.
\newblock {\em arXiv preprint arXiv:2403.04613}, 2024.

\bibitem{gupta2020nested}
Chirag Gupta, Arun~K. Kuchibhotla, and Aaditya Ramdas.
\newblock Nested conformal prediction and quantile out-of-bag ensemble methods.
\newblock {\em Pattern Recognition}, 127:108496, 2022.

\bibitem{lei2013distribution}
Jing Lei, James Robins, and Larry Wasserman.
\newblock Distribution-free prediction sets.
\newblock {\em {Journal of the American Statistical Association}}, 108(501):278--287, 2013.

\bibitem{hernan2010causal}
Miguel~A Hern{\'a}n and James~M Robins.
\newblock Causal inference, 2010.

\bibitem{brand2010benefits}
Jennie~E Brand and Yu~Xie.
\newblock Who benefits most from college? evidence for negative selection in heterogeneous economic returns to higher education.
\newblock {\em American sociological review}, 75(2):273--302, 2010.

\bibitem{morgan2001counterfactuals}
Stephen~L Morgan.
\newblock Counterfactuals, causal effect heterogeneity, and the catholic school effect on learning.
\newblock {\em Sociology of education}, pages 341--374, 2001.

\bibitem{xie2012estimating}
Yu~Xie, Jennie~E Brand, and Ben Jann.
\newblock Estimating heterogeneous treatment effects with observational data.
\newblock {\em Sociological methodology}, 42(1):314--347, 2012.

\bibitem{florens2008identification}
Jean-Pierre Florens, James~J Heckman, Costas Meghir, and Edward Vytlacil.
\newblock Identification of treatment effects using control functions in models with continuous, endogenous treatment and heterogeneous effects.
\newblock {\em Econometrica}, 76(5):1191--1206, 2008.

\bibitem{lei2021conformal}
Lihua Lei and Emmanuel~J Cand{\`e}s.
\newblock Conformal inference of counterfactuals and individual treatment effects.
\newblock {\em Journal of the Royal Statistical Society Series B: Statistical Methodology}, 83(5):911--938, 2021.

\bibitem{hill2011bayesian}
Jennifer~L Hill.
\newblock Bayesian nonparametric modeling for causal inference.
\newblock {\em Journal of Computational and Graphical Statistics}, 20(1):217--240, 2011.

\bibitem{facebook_data}
Facebook comment volume data set.
\newblock \url{https://archive.ics.uci.edu/ml/datasets/Facebook+Comment+Volume+Dataset}.
\newblock Accessed: January, 2019.

\bibitem{bio_data}
bio.
\newblock Physicochemical properties of protein tertiary structure data set.
\newblock \url{https://archive.ics.uci.edu/ml/datasets/Physicochemical+Properties+of+Protein+Tertiary+Structure}.
\newblock Accessed: January, 2019.

\bibitem{house_data}
house.
\newblock House sales in king county, {USA}.
\newblock \url{https://www.kaggle.com/harlfoxem/housesalesprediction/metadata}.
\newblock Accessed: July, 2021.

\bibitem{meps19_data}
meps\_19.
\newblock Medical expenditure panel survey, panel 19.
\newblock \url{https://meps.ahrq.gov/mepsweb/data_stats/download_data_files_detail.jsp?cboPufNumber=HC-181}.
\newblock Accessed: January, 2019.

\bibitem{blog_data}
blog\_data.
\newblock Blogfeedback data set.
\newblock \url{https://archive.ics.uci.edu/ml/datasets/BlogFeedback}.
\newblock Accessed: January, 2019.

\bibitem{wei2022learning}
Jiaheng Wei, Zhaowei Zhu, Hao Cheng, Tongliang Liu, Gang Niu, and Yang Liu.
\newblock Learning with noisy labels revisited: A study using real-world human annotations.
\newblock In {\em International Conference on Learning Representations}, 2022.

\bibitem{krizhevsky2009learning}
Alex Krizhevsky, Geoffrey Hinton, et~al.
\newblock Learning multiple layers of features from tiny images.
\newblock 2009.

\bibitem{label_noise}
Bat-Sheva Einbinder, Shai Feldman, Stephen Bates, Anastasios~N Angelopoulos, Asaf Gendler, and Yaniv Romano.
\newblock Conformal prediction is robust to dispersive label noise.
\newblock In {\em Conformal and Probabilistic Prediction with Applications}, pages 624--626. PMLR, 2023.

\bibitem{rubin1978bayesian}
Donald~B Rubin.
\newblock Bayesian inference for causal effects: The role of randomization.
\newblock {\em The Annals of statistics}, pages 34--58, 1978.

\bibitem{rosenbaum1983central}
Paul~R Rosenbaum and Donald~B Rubin.
\newblock The central role of the propensity score in observational studies for causal effects.
\newblock {\em Biometrika}, 70(1):41--55, 1983.

\bibitem{imbens2015causal}
Guido~W Imbens and Donald~B Rubin.
\newblock {\em Causal inference in statistics, social, and biomedical sciences}.
\newblock Cambridge university press, 2015.

\bibitem{twins}
Douglas Almond, Kenneth~Y. Chay, and David~S. Lee.
\newblock The costs of low birth weight.
\newblock {\em The Quarterly Journal of Economics}, 120(3):1031--1083, 2005.

\bibitem{louizos2017causal}
Christos Louizos, Uri Shalit, Joris~M Mooij, David Sontag, Richard Zemel, and Max Welling.
\newblock Causal effect inference with deep latent-variable models.
\newblock {\em Advances in neural information processing systems}, 30, 2017.

\bibitem{yeager2019national}
David~S Yeager, Paul Hanselman, Gregory~M Walton, Jared~S Murray, Robert Crosnoe, Chandra Muller, Elizabeth Tipton, Barbara Schneider, Chris~S Hulleman, Cintia~P Hinojosa, et~al.
\newblock A national experiment reveals where a growth mindset improves achievement.
\newblock {\em Nature}, 573(7774):364--369, 2019.

\bibitem{hendrycks2019robustness}
Dan Hendrycks and Thomas Dietterich.
\newblock Benchmarking neural network robustness to common corruptions and perturbations.
\newblock {\em Proceedings of the International Conference on Learning Representations}, 2019.

\bibitem{carvalho2019assessing}
Carlos Carvalho, Avi Feller, Jared Murray, Spencer Woody, and David Yeager.
\newblock Assessing treatment effect variation in observational studies: Results from a data challenge.
\newblock {\em Observational Studies}, 5(2):21--35, 2019.

\bibitem{adam}
Diederik~P. Kingma and Jimmy Ba.
\newblock Adam: A method for stochastic optimization.
\newblock In {\em 3rd International Conference on Learning Representations}, 2015.

\bibitem{xgboost}
Tianqi Chen and Carlos Guestrin.
\newblock {XGBoost}: A scalable tree boosting system.
\newblock In {\em Proceedings of the 22nd ACM SIGKDD International Conference on Knowledge Discovery and Data Mining}, KDD '16, pages 785--794, New York, NY, USA, 2016. ACM.

\bibitem{scikit-learn}
F.~Pedregosa, G.~Varoquaux, A.~Gramfort, V.~Michel, B.~Thirion, O.~Grisel, M.~Blondel, P.~Prettenhofer, R.~Weiss, V.~Dubourg, J.~Vanderplas, A.~Passos, D.~Cournapeau, M.~Brucher, M.~Perrot, and E.~Duchesnay.
\newblock Scikit-learn: Machine learning in {P}ython.
\newblock {\em Journal of Machine Learning Research}, 12:2825--2830, 2011.

\bibitem{pytorch}
Adam Paszke, Sam Gross, Francisco Massa, Adam Lerer, James Bradbury, Gregory Chanan, Trevor Killeen, Zeming Lin, Natalia Gimelshein, Luca Antiga, Alban Desmaison, Andreas Kopf, Edward Yang, Zachary DeVito, Martin Raison, Alykhan Tejani, Sasank Chilamkurthy, Benoit Steiner, Lu~Fang, Junjie Bai, and Soumith Chintala.
\newblock Pytorch: An imperative style, high-performance deep learning library.
\newblock In {\em Advances in Neural Information Processing Systems 32}, pages 8024--8035. Curran Associates, Inc., 2019.

\bibitem{he2016deep}
Kaiming He, Xiangyu Zhang, Shaoqing Ren, and Jian Sun.
\newblock Deep residual learning for image recognition.
\newblock In {\em Proceedings of the IEEE Conference on Computer Vision and Pattern Recognition}, pages 770--778, 2016.

\end{thebibliography}

\newpage

\appendix

\section{Theoretical results}\label{sec:proofs}

\subsection{Proof of Proposition~\ref{prop:2staged_validity}}\label{sec:2staged_proof}

\begin{proof}
For the sake of this proof, we re-define the scores quantile $Q^\texttt{WCP}_\text{conservative}$ from Section~\ref{sec:baseline_two_staged} as a function of a test weight $\omega$:
\begin{equation}\label{eq:general_q_baseline}
Q(\omega) := \text{Quantile}\left(1-\alpha + \beta; \sum_{j\in\mathcal{I}_2^{\text{uc}}} p_j(\omega) \delta_{S_j} + p_{n+1}(\omega) \delta_{\infty}\right),
\end{equation}
where $p_j(\omega), p_{n+1}(\omega)$ are defined as:
\begin{equation}
p_j(\omega) = \frac{w_j} {\sum_{k\in\mathcal{I}_2^{\text{uc}}} w_k + \omega}, p_{n+1}(\omega) =\frac{\omega} {\sum_{k\in\mathcal{I}_2^{\text{uc}}} w_k + \omega}.
\end{equation}
For ease of notation, we denote the test weight by $w_{n+1}= w(Z_{n+1})$ and its conservative counterpart by $\tilde{w}_{n+1}=w^\text{conservative}_{n+1}$, which is defined in~\eqref{eq:w_conservative}.
Note that by the definition of $Q$, we get $Q^\texttt{WCP}_\text{conservative} \equiv Q(\tilde{w}_{n+1})$.
Since the observed calibration points $\{(X_i(M_i), Z_i)\}_{i\in\mathcal{I}_2}$ and the test point $(X_{n+1}(0), Z_{n+1})$ are exchangeable (we assume that $X_i(0)=X_i(1)$), then \texttt{CP}~\cite{vovk2005algorithmic} guarantees that the prediction set $C^{Z}(X^\text{test})$ satisfies the coverage requirement:
\begin{equation}
    \mathbb{P}(Z_{n+1} \in C^{Z}(X^\text{test})) \geq 1- \beta,
\end{equation}
and therefore:
\begin{equation}
    \mathbb{P}(w(Z_{n+1}) \in \{w(z): z\in C^{Z}(X^\text{test}) \}) \geq 1- \beta.
\end{equation}
Following this, we get: $\mathbb{P}(w(Z_{n+1}) \leq \tilde{w}_{n+1}) \geq 1- \beta$.
Assuming $(X(0),Y(0)) \indep M \mid Z$, there is a covariate shift between the calibration and test samples, where the covariates are the privileged information $Z$. Specifically, the calibration covariates $\{Z_i\}_{i \in \mathcal{I}_2^{\text{uc}}}$ are drawn from $P_{Z\mid M=0}$ while the test covariates $Z^\text{test}$ are drawn from $P_{Z}$. Importantly, both the calibration and response response variables have the same distribution conditional on $Z$: $P_{Y\mid Z}$.
Thus, \cite[Theorem 1]{tibshirani2019conformal} states that:
\begin{equation}
     \mathbb{P}\left( Y^\text{test} \in \left\{y : \mathcal{S}(X^\text{test}, y) \leq Q(w_{n+1}) \right\}\right) \geq 1-\alpha +\beta .
\end{equation}
We note that Lemma~\ref{lem:non_decreasing_quantile} states that $Q(\omega)$ is non-decreasing, i.e., $\omega_1 \geq \omega_2 \Rightarrow Q(\omega_1) \geq Q(\omega_2)$.
Finally, we combine everything together to get: 
\begin{equation}
\begin{split}
     \mathbb{P}\left( Y^\text{test} \in C^\ttbaseline(X^\text{test})\right) &=\mathbb{P}\left( Y^\text{test} \in \left\{y : \mathcal{S}(X^\text{test}, y) \leq Q(\tilde{w}_{n+1}) \right\}\right) \\
     & = \mathbb{P}\left( \mathcal{S}(X^\text{test},Y^\text{test}) \leq Q(\tilde{w}_{n+1}) \right)  \\
     &\geq \mathbb{P}\left( \mathcal{S}(X^\text{test},Y^\text{test}) \leq Q(\tilde{w}_{n+1}), \tilde{w}_{n+1} \geq w_{n+1}\right) \\
     &\geq \mathbb{P}\left( \mathcal{S}(X^\text{test},Y^\text{test}) \leq Q(w_{n+1}), \tilde{w}_{n+1} \geq {w}_{n+1}\right) \\
     &= 1-\mathbb{P}\left(\mathcal{S}(X^\text{test},Y^\text{test}) > Q({w}_{n+1}) \text{ or } \tilde{w}_{n+1} < {w}_{n+1}\right) \\
     &\geq 1-\mathbb{P}\left( \mathcal{S}(X^\text{test},Y^\text{test}) > Q({w}_{n+1})\right) -\mathbb{P}( \tilde{w}_{n+1} < {w}_{n+1}) \\
     &\geq 1- \left(\alpha - \beta \right) - \left(\beta\right) \\
     &= 1- \alpha.
\end{split}
\end{equation}

\end{proof}

\subsection{Proof of Theorem~\ref{thm:validity}}\label{sec:validity_proof}

\begin{proof}

For the sake of this proof, we re-define the scores quantile $Q$ from~\eqref{eq:Q_i} as a function of a test weight $\omega$, similarly to~\eqref{eq:general_q_baseline}:
\begin{equation}\label{sec:general_q_pcp}
Q(\omega) := \text{Quantile}\left(1-\alpha + \beta; \sum_{j\in\mathcal{I}_2^{\text{uc}}} p_j(\omega) \delta_{S_j} + p_{n+1}(\omega) \delta_{\infty}\right),
\end{equation}
where $p_j(\omega), p_{n+1}(\omega)$ are defined as:
\begin{equation}
p_j(\omega) = \frac{w_j} {\sum_{k\in\mathcal{I}_2^{\text{uc}}} w_k + \omega}, \ \ p_{n+1}(\omega) =\frac{\omega} {\sum_{k\in\mathcal{I}_2^{\text{uc}}} w_k + \omega}.
\end{equation}
Note that by the definition of $Q$, we get $Q_i \equiv Q(w_i)$. Furthermore:
\begin{align}
Q^\ttmethod &\equiv \text{Quantile}\left(1-\beta; \sum_{i\in\mathcal{I}_2} \frac{1}{|\mathcal{I}_2|+1}\delta_{Q_i} + \frac{1}{|\mathcal{I}_2|+1}\delta_{\infty} \right) \\ 
&= \text{Quantile}\left(1-\beta; \sum_{i\in\mathcal{I}_2} \frac{1}{|\mathcal{I}_2|+1}\delta_{Q(w_i)} + \frac{1}{|\mathcal{I}_2|+1}\delta_{\infty} \right).
\end{align}
Since $Q(\omega)$ is a non-decreasing function of $\omega$, as proved in Lemma~\ref{lem:non_decreasing_quantile}, we get:
\begin{equation}
Q^\ttmethod = Q\left(\text{Quantile}\left(1-\beta; \sum_{i\in\mathcal{I}_2} \frac{1}{|\mathcal{I}_2|+1}\delta_{w_i}+ \frac{1}{|\mathcal{I}_2|+1}\delta_{\infty} \right) \right).
\end{equation}
Thus, $Q^\ttmethod$ can be considered as if it was computed from the following weight:
\begin{equation}
\tilde{w}_{n+1} := \text{Quantile}\left(1-\beta; \sum_{i\in\mathcal{I}_2} \frac{1}{|\mathcal{I}_2|+1}\delta_{w_{i}} +\frac{1}{|\mathcal{I}_2|+1}\delta_{\infty} \right),   
\end{equation}
in the sense that $Q^\ttmethod =Q(\tilde{w}_{n+1})$. Therefore:
\begin{equation}
C^\ttmethod(X^\text{test}) :=\left\{y : \mathcal{S}(X^\text{test}, y) \leq Q^\ttmethod \right\} = \left\{y : \mathcal{S}(X^\text{test}, y) \leq Q(\tilde{w}_{n+1}) \right\}.
\end{equation}
The true weight of the $n+1$ sample is: $w_{n+1}=\frac{\mathbb{P}(M=0)}{\mathbb{P}(M=0 \mid Z=Z_{n+1})}$. 
Now, since $\{ Z_i\}_{i\in \mathcal{I}_2 \cup \{n+1\}}$ are exchangeable, then $\{w_i\}_{i\in \mathcal{I}_2 \cup \{n+1\}}$ are exchangeable and thus~\cite[Lemma 2]{romano2019conformalized} states that:
\begin{equation}
    \mathbb{P}(w_{n+1} \leq \tilde{w}_{n+1}) \geq 1- \beta.
\end{equation}
Assuming $(X(0),Y(0)) \indep M \mid Z$, there is a covariate shift between the calibration and test samples, where the covariates are the privileged information $Z$. Specifically, the uncorrupted calibration covariates $\{Z_i\}_{i \in \mathcal{I}_2^{\text{uc}}}$ are drawn from $P_{Z\mid M=0}$ while the test covariates $Z^\text{test}$ are drawn from $P_{Z}$. Importantly, both the calibration and response response variables have the same distribution conditional on $Z$: $P_{Y\mid Z}$. Thus,~\cite[Theorem 1]{tibshirani2019conformal} states that:
\begin{equation}
     \mathbb{P}\left( Y^\text{test} \in \left\{y : \mathcal{S}(X^\text{test}, y) \leq Q(w_{n+1}) \right\}\right) \geq 1-\alpha +\beta.
\end{equation}
Note that Lemma~\ref{lem:non_decreasing_quantile} states that $Q(\omega)$ is non-decreasing, i.e., $\omega_1 \geq \omega_2 \Rightarrow Q(\omega_1) \geq Q(\omega_2)$.
Finally, we combine all together and get:
\begin{equation}
\begin{split}
     \mathbb{P}\left( Y^\text{test} \in C^\ttmethod(X^\text{test})\right) &=\mathbb{P}\left( Y^\text{test} \in \left\{y : \mathcal{S}(X^\text{test}, y) \leq Q(\tilde{w}_{n+1}) \right\}\right) \\
     & = \mathbb{P}\left( \mathcal{S}(X^\text{test},Y^\text{test}) \leq Q(\tilde{w}_{n+1})\right)  \\
     &\geq \mathbb{P}\left( \mathcal{S}(X^\text{test},Y^\text{test}) \leq Q(\tilde{w}_{n+1}), \tilde{w}_{n+1} \geq w_{n+1}\right) \\
     &\geq \mathbb{P}\left( \mathcal{S}(X^\text{test},Y^\text{test}) \leq Q(w_{n+1}), \tilde{w}_{n+1} \geq {w}_{n+1}\right) \\
     &= 1-\mathbb{P}\left( \mathcal{S}(X^\text{test},Y^\text{test}) > Q({w}_{n+1}) \text{ or } \tilde{w}_{n+1} < {w}_{n+1}\right) \\
     &\geq 1-\mathbb{P}\left( \mathcal{S}(X^\text{test},Y^\text{test}) > Q({w}_{n+1})\right) -\mathbb{P}( \tilde{w}_{n+1} < {w}_{n+1}) \\
     &\geq 1- \left(\alpha - \beta\right) - \left(\beta\right) \\
     &= 1- \alpha.
\end{split}
\end{equation}
\end{proof}

\subsection{Proof of Theorem~\ref{thm:validity_j}}\label{sec:validity_pcp_loo}

\begin{proof}
For the sake of this proof, we re-define the prediction set as a function of the test weight $\omega \in [0,1]$:
\begin{equation}
C^\ttmethodjp_\omega(x)=\left\{y \in \mathcal{Y}: \sum_{i\in \mathcal{I}^\text{uc}}p_{i}(\omega)\mathbbm{1}\{S_i < \mathcal{S}(x, y;\hat{f}^{-i})\} < 1-\gamma \right\},
\end{equation}
where $\gamma=\alpha - \frac{1}{2}\beta$, and $p_i(\omega), p_{n+1}(\omega)$ are defined as:
\begin{equation}
p_i(\omega) = \frac{w_i} {\sum_{k\in\mathcal{I}^{\text{uc}}} w_k + \omega}, p_{n+1}(\omega) =\frac{\omega} {\sum_{k\in\mathcal{I}^{\text{uc}}} w_k + \omega}.
\end{equation}
The prediction set generated by \ttmethodjp is: $C^\ttmethodjp_{\tilde{w}_{n+1}}(x)$, where:
\begin{equation}
\tilde{w}_{n+1} := \text{Quantile}\left(1-\beta; \sum_{i=1}^{n} \frac{1}{n+1}\delta_{w_{i}} +\frac{1}{n+1}\delta_{\infty}\right).
\end{equation}
Therefore, our goal is to show that this prediction set covers the response variable at the desired coverage rate:
\begin{equation}\label{eq:pcp_jp_valid_cov}
    \mathbb{P}(Y^\text{test} \in C^\ttmethodjp_{\tilde{w}_{n+1}}(X^\text{test})) \geq 1- 2\alpha.
\end{equation}
The proof consists of two steps. In the first step, we show that $C^\ttmethodjp_{w_{n+1}}(X^\text{test}))$ achieves a valid marginal coverage, namely:
\begin{equation}
    \mathbb{P}(Y^\text{test} \in C^\ttmethodjp_{w_{n+1}}(X^\text{test})) \geq 1- 2\gamma.
\end{equation}
In the second step, we show that $C^\ttmethodjp_{\tilde{w}_{n+1}}(X^\text{test}))$ is a super set of $C^\ttmethodjp_{{w}_{n+1}}(X^\text{test}))$ with a high probability. From these two steps, we will conclude~\eqref{eq:pcp_jp_valid_cov}.

We define the matrix of residuals, similarly to \cite{barber2019jackknife, prinster2022jaws}, denoted by $R\in \mathbb{R}^{(n+1)\times(n+1)}$, with entries:
\begin{equation}
R_{i,j} = \begin{cases} 
      +\infty & i=j, \\
      \mathcal{S}(X_i(0), Y_i(0), \hat{f}^{-i,j}) & i \ne j,
   \end{cases}
\end{equation}
where $\hat{f}^{-i,j}$ is the model $\hat{f}$ fitted on all samples except with the points $i$ and $j$ removed, namely, $\{1,...,n+1\}-\{i,j\}$.
For simplicity, we denote $\tilde{w}_i(\omega):=w_i$ for $i\in \{1,..,n\}$ and $\tilde{w}_{n+1}(\omega)=\omega$. We follow the definition in \cite{prinster2022jaws}
of “strange” points $\mathcal{G}(\omega) \subseteq \mathcal{I}^\text{uc} \cup \{ n + 1\}$:
\begin{equation}
\mathcal{G}(\omega) = \left\{ i \in \mathcal{I}^\text{uc} \cup \{ n + 1\} : \tilde{w}_i(\omega) > 0,
\sum_{j\in \mathcal{I}^\text{uc} \cup \{ n + 1\}} p_j(\omega) \mathbbm{1}\left\{R_{ij} > R_{ji}\right\} \geq 1-\gamma \right\}.
\end{equation}

We begin by showing that $Y^\text{test} \in C^\ttmethodjp_{w_{n+1}}(X^\text{test})) \Rightarrow n+1 \notin \mathcal{G}(w_{n+1})$.
Suppose that $Y^\text{test} \in C^\ttmethodjp_{w_{n+1}}(X^\text{test})$. Then, by definition of $C^\ttmethodjp_{w_{n+1}}(X^\text{test}))$:
\begin{equation}
\begin{split}
1-\gamma & > \sum_{j \in \mathcal{I}^\text{uc}}p_{j}(w_{n+1})\mathbbm{1}\{S_j < \mathcal{S}(X^\text{test}, Y^\text{test};\hat{f}^{-j})\}\\
&= \sum_{j\in\mathcal{I}^\text{uc}}p_{j}(w_{n+1})\mathbbm{1}\{R_{j,n+1} < R_{n+1,j}\}\\
&= \sum_{j \in \mathcal{I}^\text{uc}}p_{j}(w_{n+1})\mathbbm{1}\{R_{n+1,j} > R_{j,n+1}\}\\
&=\sum_{j\in \mathcal{I}^\text{uc} \cup \{ n + 1\} }p_{j}(w_{n+1})\mathbbm{1}\{R_{n+1,j} > R_{j,n+1}\}.\\
\end{split}
\end{equation}
Therefore: $n+1 \notin \mathcal{G}(w_{n+1})$, by the definition of $\mathcal{G}$. We deduce that: $\mathbb{P}(Y^\text{test} \in C^\ttmethodjp_{w_{n+1}}(X^\text{test})) \geq \mathbb{P}(n+1 \notin \mathcal{G}(w_{n+1}))$.
In~\cite{prinster2022jaws} it is shown that:
\begin{equation}
    \mathbb{P}(n+1 \notin \mathcal{G}(w_{n+1})) \geq 1- 2\gamma.
\end{equation}
Thus:
\begin{equation}
    \mathbb{P}(Y^\text{test} \in C^\ttmethodjp_{w_{n+1}}(X^\text{test})) \geq \mathbb{P}(n+1 \notin \mathcal{G}(w_{n+1})) \geq 1- 2\gamma.
\end{equation}
We now turn to the second step of the proof. Similarly to Lemma~\ref{lem:non_decreasing_quantile}, we now show that $C^\ttmethodjp_{\omega}(x)$ is a monotonic function of $\omega$, i.e., if $\omega_1 \geq \omega_2$ then $C^\ttmethodjp_{\omega_2}(x)\subseteq C^\ttmethodjp_{\omega_1}(x)$. Suppose that $y \in C^\ttmethodjp_{\omega_2}(x)$. Then:
\begin{equation}
\sum_{i\in \mathcal{I}^\text{uc}}p_{i}(\omega_2)\mathbbm{1}\{S_i < \mathcal{S}(x, y;\hat{f}^{-i})\} < 1-\gamma.
\end{equation}
Since $\omega_1 \geq \omega_2$, we get $p_i(\omega_1) \leq p_i(\omega_2)$ for all $i\in\{1,...,n\}$. Therefore:
\begin{equation}
\sum_{i\in \mathcal{I}^\text{uc}}p_{i}(\omega_1)\mathbbm{1}\{S_i < \mathcal{S}(x, y;\hat{f}^{-i})\}  \leq \sum_{i\in \mathcal{I}^\text{uc}}p_{i}(\omega_2)\mathbbm{1}\{S_i < \mathcal{S}(x, y;\hat{f}^{-i})\} < 1-\gamma.
\end{equation}
Meaning that $y\in C^\ttmethodjp_{\omega_1}(x)$ as well.
Lastly, we recall that according to~\cite[Lemma 2]{romano2019conformalized} the weight $\tilde{w}_{n+1}$ satisfies:
\begin{equation}
    \mathbb{P}(w_{n+1} \leq \tilde{w}_{n+1}) \geq 1- \beta.
\end{equation}
Finally, we combine everything together to get:
\begin{equation}
\begin{split}
     \mathbb{P}\left( Y^\text{test} \in C_{\tilde{w}_{n+1}}^\ttmethodjp(X^\text{test})\right) &\geq \mathbb{P}\left( Y^\text{test} \in C_{\tilde{w}_{n+1}}^\ttmethodjp(X^\text{test}), \tilde{w}_{n+1} \geq w_{n+1} \right) \\
     &\geq \mathbb{P}\left( Y^\text{test} \in C_{w_{n+1}}^\ttmethodjp(X^\text{test}), \tilde{w}_{n+1} \geq w_{n+1} \right) \\
      &= 1- \mathbb{P}\left( Y^\text{test} \notin C_{w_{n+1}}^\ttmethodjp(X^\text{test})\text{ or } \tilde{w}_{n+1} < w_{n+1} \right)\\
      &\geq 1- \mathbb{P}\left( Y^\text{test} \notin C_{w_{n+1}}^\ttmethodjp(X^\text{test})\right) -\mathbb{P}\left(\tilde{w}_{n+1} < w_{n+1} \right)\\
      &\geq 1- (2\gamma) -\left(\beta  \right)\\
      &\geq 1- 2\left(\alpha - \frac{1}{2}\beta \right) -\left(\beta \right)\\
     &= 1- 2\alpha.
\end{split}
\end{equation}


\end{proof}

\subsection{Lemma about weighted quantiles}

\begin{lemma}\label{lem:non_decreasing_quantile}
Suppose that $w_i, S_i \in \mathbb{R}$ for $1\leq i\leq n$. Further, suppose that $\gamma \in [0,1]$, and $\mathcal{I} \subseteq \{1,...,n \}$.
Denote the normalized weights $p_i(\omega), p_{n+1}(\omega)$ as:
\begin{equation}
p_i(\omega) = \frac{w_i} {\sum_{k\in\mathcal{I}} w_k + \omega}, p_{n+1}(\omega) =\frac{\omega} {\sum_{k\in\mathcal{I}} w_k + \omega}.
\end{equation}
Then, the weighted quantile:
\begin{equation}
Q(\omega) := \text{Quantile}\left(\gamma; \sum_{j\in\mathcal{I}} p_i(\omega) \delta_{S_j} + p_{n+1}(\omega) \delta_{\infty}\right),
\end{equation}
is a non-decreasing function of $\omega$, i.e.,
\begin{equation}
\omega_1 \geq \omega_2 \Rightarrow Q(\omega_1) \geq Q(\omega_2). 
\end{equation}
\end{lemma}
\begin{proof}
Without loss of generality, suppose that $\{S_i\}_{i=1}^{n}$ are sorted in an increasing order, i.e., $S_{i+1} \geq S_i$. For the sake of this proof, we define $S_{n+1}=\infty$. For ease of notation, denote $\mathcal{I}' = \mathcal{I} \cup \{n+1\}$ and $\mathcal{I}_i = \mathcal{I}' \cap \{1,...,i\}$. The formal definition of $Q(\omega)$ is:
\begin{equation}
Q(\omega) = S_{\min \left\{ i\in \mathcal{I}' : \sum_{k\in \mathcal{I}_i}{p_{k}(\omega)}\geq \gamma \right\}}.
\end{equation}
Suppose that $\omega_1 \geq \omega_2$. Then, by the definition of $p_i$, we get that for all $i\in \mathcal{I}$:
\begin{equation}
p_i(\omega_1) \leq p_i (\omega_2).
\end{equation}
Therefore for all $i\in \mathcal{I}$:
\begin{equation}
\sum_{k\in\mathcal{I}_i}{p_{k}(\omega_1)} \leq \sum_{k\in\mathcal{I}_i}{p_{k}(\omega_2)}.
\end{equation}
For $i={n+1}$ the above equation is trivially satisfied as $\sum_{k\in\mathcal{I}_{n+1}}{p_{k}(\omega_1)} = \sum_{k\in\mathcal{I}_{n+1}}{p_{k}(\omega_2)}=1.$ Thus,
\begin{equation}
\begin{split}
    \min &\left\{ i\in \{1,...,{n+1}\} : \sum_{k\in \mathcal{I}_i}{p_{k}(\omega_1)}\geq \gamma \right\} \geq \\
    &\min \left\{ i\in \{1,...,{n+1}\} : \sum_{k\in \mathcal{I}_i}{p_{k}(\omega_2)}\geq \gamma \right\}.
    \end{split}
\end{equation}
Therefore,
\begin{equation}
\begin{split}
    S_{\min \left\{ i\in \{1,...,{n+1}\} : \sum_{k\in \mathcal{I}_i}{p_{k}(\omega_1)}\geq \gamma \right\}} \geq
    S_{\min \left\{ i\in \{1,...,{n+1}\} : \sum_{k\in \mathcal{I}_i}{p_{k}(\omega_2)}\geq \gamma \right\}},
    \end{split}
\end{equation}
and finally,
\begin{equation}
Q(\omega_1) \geq Q(\omega_2).
\end{equation}
\end{proof}

\subsection{Relaxing the conditional independence assumption}\label{sec:relaxing_cond_indep_assumption}

In this section, we present two relaxations for conditional independence assumption, $(X(0),Y(0)) \indep M \mid Z$, in Theorem~\ref{thm:validity}. The first relaxation allows $X(0)$ to depend on $M$ given the PI at the expense of using the following weights~$\mathbb{P}(M\mid X=x,Z=x)$. The second result relaxes the conditional independence to hold approximately, up to some error, denoted by $\varepsilon$. We begin with formalizing the first result and then turn to the second one.

\begin{theorem}[Robustness of \ttmethod to dependence of the features and corruption indicator]\label{thm:relaxed_validity_xz}
Suppose that $\{({X}_i(0),{X}_i(1), {Y}_i(0),{Y}_i(1),Z_i,  M_i)\}_{i=1}^{n+1}$ are exchangeable, $(Y(0) \indep M) \mid (X,Z)$, the features are uncorrupted, $X(0)=X(1)$, and $P_{X(0),Z}$ is absolutely continuous with respect to $P_{X(0),Z \mid M=0}$. 
Denote by $C^\ttmethod(X^\textup{test})$ the prediction set constructed according to Algorithm~\ref{alg:weighted2} with the weights:
\begin{equation}
w_i := \frac{\mathbb{P}(M=0)}{\mathbb{P}(M=0 \mid X(0)=X_i(0), Z=Z_i)}
\end{equation}
Then, this prediction set achieves the desired coverage rate:
\begin{equation}
\mathbb{P}(Y^\textup{test} \in C^\ttmethod(X^\textup{test})) \geq 1-\alpha.
\end{equation}
\end{theorem}
\begin{proof}
This result follows from the proof of Theorem~\ref{thm:validity}, except for the following changes. Here, we get that then there is a covariate shift between the calibration and test samples, where the covariates are $X(0),Z$. Therefore, the following weight $w_{n+1}$ of this setup:
\begin{equation}
w_{n+1} = \frac{\mathbb{P}(M=0)}{\mathbb{P}(M=0 \mid X(0)=X_{n+1}(0), Z=Z_{n+1})}
\end{equation}
satisfies:
\begin{equation}
     \mathbb{P}\left( Y^\text{test} \in \left\{y : \mathcal{S}(X^\text{test}, y) \leq Q(w_{n+1}) \right\}\right) \geq 1-\alpha +\beta,
\end{equation}
where $Q$ is defined as in Appendix~\ref{sec:validity_proof}. The rest of the proof is as in Appendix~\ref{sec:validity_proof}.
\end{proof}

We now turn to present an initial extension of Theorem~\ref{thm:validity} to a setting where the conditional independence assumption is not fully satisfied. For the simplicity of this extension, we assume $X(0) \indep M \mid Z=z$.

We note that the independence assumption $ Y(0) \indep  M \mid Z=z$ is equivalent to assuming that the density of $Y(0) \mid M=m, Z=z$ is the same for $m \in \{0,1\}$, formally:

$$f_{Y(0) \mid M=0, X=x,Z=z}(y; 0,x,z) = f_{Y(0) \mid M=1,X=x, Z=z}(y; 1,x,z). $$

In this extension, we relax this assumption and instead require that $\forall x\in \mathcal{X}$, there exists $\varepsilon_x \in\mathbb{R}$ such that the difference between the two densities is bounded by $\varepsilon_x$:
$$\forall y\in\mathcal{Y}, z\in\mathcal{Z}: | f_{Y(0) \mid M=0, X=x,Z=z}(y; 0,x,z) - f_{Y(0) \mid M=1,X=x, Z=z}(y; 1,x,z)  | \leq \varepsilon_x .$$

\begin{theorem}[Robustness of \ttmethod to conditional independence violation]\label{thm:relaxed_validity_eps}
Suppose that $\{({X}_i(0),{X}_i(1), {Y}_i(0),{Y}_i(1),Z_i,  M_i)\}_{i=1}^{n+1}$ are exchangeable, the features are independent of the corruption indicator given the PI, $X(0) \indep M \mid Z$, the probability $P_{Z}$ is absolutely continuous with respect to $P_{Z \mid M=0}$, and $\forall x\in\mathcal{X}$ there exists $\varepsilon_x \in\mathbb{R} $ such that:
$$\forall y\in\mathcal{Y}, z\in\mathcal{Z}| f_{Y(0) \mid M=0, X=x,Z=z}(y; 0,x,z) - f_{Y(0) \mid M=1,X=x, Z=z}(y; 1,x,z)  | \leq \varepsilon_x .$$
Then, the coverage rate of the prediction set $C^\ttmethod(X^\textup{test})$ constructed according to Algorithm~\ref{alg:weighted2} is lower bounded by:
\begin{equation}
\mathbb{P}(Y^\textup{test} \in C^\ttmethod(X^\textup{test})) \geq 1-\alpha- \mathbb{E}_{X,Z}[ | C^\ttmethod (X) | \varepsilon_X \mathbb{P}(M=1\mid X,Z)] .
\end{equation}
\end{theorem}

\begin{proof}
We define by $V$ a variable that is drawn from:
$$ (Z,V) \sim P_Z \cross P_{Y(0) \mid Z=z, M=0} .$$
By the definition of $V$, and according to Theorem~\ref{thm:validity}, \ttmethod covers $V$ with $1-\alpha$ probability:
$$ \mathbb{P}(V \in C^\ttmethod (X)) \geq 1-\alpha.$$

Notice that $V \mid X=x, Z=z$ equals in distribution to $ Y(0) \mid M=0, X=x,Z=z$ by definition. We denote the coverage rate of \ttmethod over $ Y(0) \mid M=0, X=x,Z=z$ by $\beta_{0,x,z}$:
$$\beta_{0,x,z} := \mathbb{P}( Y(0) \in C^\ttmethod (X) \mid X=x,Z=z, M=0)=\mathbb{P}( V \in C^\ttmethod (X) \mid X=x,Z=z).$$

Similarly, we denote the coverage rate of \ttmethod over $Y(0) \mid M=1, X=x,Z=z$ by $\beta_{1,x,z}$:
$$\beta_{1,x,z} := \mathbb{P}( Y(0) \in C^\ttmethod (X) \mid M=1, X=x,Z=z) = \int_{y \in C^\ttmethod (x) }  f_{Y(0) \mid M=1,X=x, Z=z}(y; 1,x,z) dy $$

This probability can be lower bounded by:

\begin{equation}
\begin{split}
 \int_{y \in C^\ttmethod (x)  }  f_{Y(0) \mid M=1,X=x, Z=z}(y; 1,x,z) dy  &\geq  \int_{y \in C^\ttmethod (x)  }  (f_{Y(0) \mid M=0,X=x, Z=z}(y;0,x,z)  - \varepsilon_x) dy
\\
& = \beta_{0,x,z} - | C^\ttmethod (x) | \varepsilon_x  
\end{split}
\end{equation}

We now compute the conditional coverage rate of \ttmethod:
\begin{equation}
\begin{split}
\mathbb{P}( Y(0) \in C^\ttmethod (X) \mid Z=z,X=x) 
& = \mathbb{P}( Y(0) \in C^\ttmethod (X) \mid M=0,Z=z,X=x)\mathbb{P}(M=0 \mid X=x,Z=z) \\
&+ \mathbb{P}( Y(0) \in C^\ttmethod (X) \mid M=1,Z=z,X=x) \mathbb{P}(M=1\mid X=x,Z=z) \\
& \geq \beta_{0,x,z}\mathbb{P}(M=0\mid X=x,Z=z) \\
&+ (\beta_{0,x,z} - | C^\ttmethod (x) | \varepsilon_x  )\mathbb{P}(M=1\mid X=x,Z=z) \\
& = \beta_{0,x,z}\mathbb{P}(M=0\mid X=x,Z=z) \\
&+ \beta_{0,x,z}\mathbb{P}(M=1\mid X=x,Z=z)  - | C^\ttmethod (x) | \varepsilon_x \mathbb{P}(M=1\mid X=x,Z=z) \\
& = \beta_{0,x,z}(\mathbb{P}(M=0,X=x,Z=z) + \mathbb{P}(M=1\mid X=x,Z=z) ) \\
&- | C^\ttmethod (x) | \varepsilon_x \mathbb{P}(M=1 \mid X=x,Z=z) \\
& = \beta_{0,x,z} - | C^\ttmethod (x) | \varepsilon_x \mathbb{P}(M=1 \mid X=x,Z=z) \\
&= \mathbb{P}(V \in C^\ttmethod (X) \mid X=x, Z=z)- | C^\ttmethod (x) | \varepsilon_x \mathbb{P}(M=1 \mid X=x,Z=z)
\end{split}
\end{equation}

By marginalizing this result we get:
\begin{equation}
\begin{split}
\mathbb{P}( Y(0) \in C^\ttmethod (X))  &= 
\int_{x \in \mathcal{X}, z\in \mathcal{Z} } { \mathbb{P}( Y(0) \in C^\ttmethod (X) \mid Z=z,X=x) f_{X, Z}(x,z)  dx dz }\\
&\geq \int_{x \in \mathcal{X}, z\in \mathcal{Z} } { [\mathbb{P}(V \in C^\ttmethod (X) \mid X=x, Z=z)] f_{X, Z}(x,z)  dx dz  }\\
&-\int_{x \in \mathcal{X}, z\in \mathcal{Z} } { [ | C^\ttmethod (x) | \varepsilon_x \mathbb{P}(M=1 \mid X=x,Z=z)] f_{X, Z}(x,z)  dx dz  } \\
&\geq 1-\alpha - \mathbb{E}_{X,Z}[ | C^\ttmethod (X) | \varepsilon_X \mathbb{P}(M=1\mid X,Z)] 
\end{split}
\end{equation}
\end{proof}

This result provides a lower bound for the coverage rate of \ttmethod in the setting where the conditional independence assumption is not exactly satisfied. Intuitively, as $\varepsilon_x$ decreases, i.e., as the two distributions $Y(0) \mid M=m, Z=z$ for $m\in \{0,1\}$ are closer to each other, the lower bound is tighter, and closer to the target level. Similarly, as the two distributions diverge, the lower bound becomes looser.

\section{Algorithms}

\subsection{Two-Staged Conformal}\label{sec:alg_baseline}

In this section, we outline the two-staged conformal prediction algorithm (\ttbaseline).

\begin{algorithm}[h]
	 \caption{Two-Staged Conformal Prediction (\ttbaseline)}
	\label{alg:baseline}
	
	\textbf{Input:}
	\begin{algorithmic}
		\State Data $({X}^\text{obs}_i, {Y}^\text{obs}_i, Z_i, M_i) \in \mathcal{X} \cross \mathcal{Y} \cross \mathcal{Z}\cross\{0,1\}, 1\leq i \leq n$.
            \State Weights $\{w_i\}_{i=1}^{n}$.
		\State Miscoverage level $\alpha \in (0,1)$.
  		\State Level $\beta \in (0,\alpha)$.
            \State An algorithm $\hat{f}^{Y}$ for $Y$.
            \State An algorithm $\hat{f}^{Z}$ for $Z$.
            \State A score function $\mathcal{S}$.
		\State A test point $X^\text{test}=x$.
	\end{algorithmic}
	
	\textbf{Process:}
	\begin{algorithmic}
            \State Randomly split $\{1,...,n\}$ into two disjoint sets $\mathcal{I}_1, \mathcal{I}_2$.
            \State Fit the base algorithm $\hat{f}^{Y}$ on $\{({X}^\text{obs}_i, {Y}^\text{obs}_i)\}_{i\in\mathcal{I}_1}$ and the algorithm $\hat{f}^{Z}$ on $\{({X}^\text{obs}_i, Z_i)\}_{i\in\mathcal{I}_1}$.
            \State Compute the scores $S^{Z}_i=\mathcal{S}({X}^\text{obs}_i, {Z}_i;\hat{f}^{Z})$, $S_i=\mathcal{S}({X}^\text{obs}_i, {Y}^\text{obs}_i;\hat{f}^{Y})$ for the calibration samples, $i\in \mathcal{I}_2$.
            \State Compute a threshold $Q^{Z}$ as the $(1-{1}/{|\mathcal{I}_2|})(1-\beta)$-th empirical quantile of the scores $\{S_i\}_{i\in \mathcal{I}_2}$.
            \State Construct a prediction set for $Z$: $C^Z(x)=\{z : \mathcal{S}(x, z, \hat{f}^{Z}) \leq Q^Z\}$.
            \State Compute the conservative test weight $w^\text{conservative}(x) := \max_{z\in C^Z(x)} w(z)$.
            \State Compute the test threshold $Q^\texttt{WCP}_\text{conservative}$ according to~\eqref{eq:wcp_q} with the clean calibration points $\{(X_i, Y_i, w(Z_i)\}_{i\in \mathcal{I}_2^\text{uc}}$, and the conservative test weight, $w^\text{conservative}(x)$, with a nominal coverage level of $1-\alpha+\beta$.
	\end{algorithmic}
	
	\textbf{Output:}
	\begin{algorithmic}
		\State Prediction set $C^\ttbaseline(x)=\{y: \mathcal{S}(x,y;\hat{f}^{Y}) \leq Q^\texttt{WCP}_\text{conservative}\}$.
	\end{algorithmic}
\end{algorithm}

\subsection{Efficient Privileged Conformal Prediction}\label{sec:efficient_pcp}

Since the complexity of Algorithm~\ref{alg:weighted2} is squared in the number of calibration samples, here, we provide a more efficient algorithm with a complexity that is linear in the number of calibration samples. The efficient version is detailed in Algorithm~\ref{alg:efficient_pcp}. The proof of Theorem~\ref{thm:validity} in Section~\ref{sec:validity_proof} shows that these two algorithms are identical, in the sense that they produce exactly the same outputs.

\begin{algorithm}[h]
	 \caption{Efficient Privileged Conformal Prediction (\ttmethod)}
	\label{alg:efficient_pcp}
	
	\textbf{Input:}
	\begin{algorithmic}
		\State Data $({X}^\text{obs}_i, {Y}^\text{obs}_i, Z_i, M_i) \in \mathcal{X} \cross \mathcal{Y} \cross \mathcal{Z}\cross\{0,1\}, 1\leq i \leq n$.
            \State Weights $\{w_i\}_{i=1}^{n}$.
		\State Miscoverage level $\alpha \in (0,1)$.
  		\State Level $\beta \in (0,\alpha)$.
            \State An algorithm $\hat{f}$.
            \State A score function $\mathcal{S}$.
		\State A test point $X^\text{test}=x$.
	\end{algorithmic}
	
	\textbf{Process:}
	\begin{algorithmic}
            \State Randomly split $\{1,...,n\}$ into two disjoint sets $\mathcal{I}_1, \mathcal{I}_2$.
            \State Fit the base algorithm $\hat{f}$ on $\{({X}^\text{obs}_i, {Y}^\text{obs}_i)\}_{i\in\mathcal{I}_1}$.
            \State Compute the scores $S_i=\mathcal{S}({X}^\text{obs}_i, {Y}^\text{obs}_i;\hat{f})$ for the calibration samples, $i\in \mathcal{I}_2^\text{uc}$.  
            \State Compute an estimated test weight based on the calibration samples: $\tilde{w}_{n+1}:= \text{Quantile}\left(1-\beta; \sum_{i \in \mathcal{I}_2} \frac{1}{n_2+1}\delta_{w_i}+\frac{1}{n_2+1}\delta_{\infty}\right)$
            \State Compute $Q^\ttmethod := Q(\tilde{w}_{n+1})$, where $Q$ is defined in~\eqref{sec:general_q_pcp}.
	\end{algorithmic}
	
	\textbf{Output:}
	\begin{algorithmic}
		\State Prediction set $C^\ttmethod(x)=\{y: \mathcal{S}(x,y;\hat{f}) \leq Q^\ttmethod\}$.
	\end{algorithmic}
\end{algorithm}

\subsection{Privileged Conformal Prediction for scarce data}\label{sec:pcp_scarce_algorithm}
In this section, we describe the leave-one-out privileged conformal prediction algorithm (\ttmethodjp), which is designed to handle scarce data more efficiently.
This procedure is summarized in Algorithm~\ref{alg:weighted2_j}.

\begin{algorithm}[h]
	 \caption{Privileged Conformal Prediction for scarce data (\ttmethodjp)}
	\label{alg:weighted2_j}
	
	\textbf{Input:}
	\begin{algorithmic}
		\State Data $({X}^\text{obs}_i,  {Y}^\text{obs}_i,Z_i, M_i) \in \mathcal{X} \cross \mathcal{Y} \cross \mathcal{Z}\cross\{0,1\}, 1\leq i \leq n$, weights $\{w_i\}_{i=1}^{n}$, miscoverage level $\alpha \in (0,1)$, level $\beta \in (0,\alpha)$, an algorithm $\hat{f}$, a score function $\mathcal{S}$, and a test point $X^\text{test}=x$.
	\end{algorithmic}
	
	\textbf{Process:} \\
        \For{$i\leftarrow 1$ \KwTo $n$}{
             Define the current training set as $\mathcal{I}=\{1,...,i-1,i+1,...,n\}$.\\
             Fit the base algorithm $\hat{f}$ on ${(X_j, Y_j)}_{j\in\mathcal{I}_1}$ to obtain $\hat{f}^{-i}$.\\
             Compute the score $S_i=\mathcal{S}({X}^\text{obs}_i, {Y}^\text{obs}_i;\hat{f}^{-i})$. \\
             
        }
        Guess the weight of the $n+1$ sample: $\tilde{w}_{n+1}=\text{Quantile}\left(1-\beta; \sum_{i\in \{1,...,n\}} \frac{1}{n+1} \delta_{w_i}+\frac{1}{n+1} \delta_{\infty} \right)$.\\
        Compute: $p_i = \frac{w_i}{\sum_{j \in \mathcal{I}^\text{uc}} w_j + \tilde{w}_{n+1}}$, for $i\in\{1,...,n\}$.\\
        Define the threshold $\gamma := \alpha - \frac{1}{2}\beta $.\\
	\textbf{Output:}
	\begin{algorithmic}
	\State Prediction set $C^\ttmethodjp(x)=\left\{y \in \mathcal{Y}: \sum_{i \in \mathcal{I}^\text{uc}}p_{i}\mathbbm{1}\{S_i < \mathcal{S}(x, y;\hat{f}^{-i})\} < 1-\gamma \right\}$.
	\end{algorithmic}
\end{algorithm}

\section{Datasets details}\label{sec:dataset_details}

\subsection{General real dataset details}\label{sec:real_dataset_details}
Table~\ref{tab:real_datasets_info} displays the size of each data set, the feature dimension, and the feature that is used as privileged information in the tabular data experiments.

\begin{table}[htbp]
\caption{Information about the real data sets.}
\setstretch{1.5}
  \centering
\scalebox{0.8}{
\centering
\begin{tabular}{cccc}
    \toprule[1.1pt]
    \textbf{Dataset Name} & \textbf{\# Samples} & $\boldsymbol{X/Z/Y}$ \textbf{Dimensions}  & $\boldsymbol{Z}$ \textbf{description} \\
    \midrule
    \textbf{facebook1~\cite{facebook_data}}  & 40948 & 52/1/1 & Number of posts comments\\
    \textbf{facebook2~\cite{facebook_data}}  & 81311 & 52/1/1 & Number of posts comments \\
    
    \textbf{Bio~\cite{bio_data}}  & 45730 & 8/1/1 & Fractional area of exposed non polar residue \\
    
    \textbf{House~\cite{house_data}}  & 21613 & 17/1/1 & Square footage of the apartments interior living space \\

    \textbf{Meps19~\cite{meps19_data}}  & 15785 & 138 /1/1 & Overall rating of feelings \\
    \textbf{Blog~\cite{blog_data}}  & 52397 & 279/1/1 & The time between the blog post publication and base-time \\
    
    \textbf{IHDP~\cite{hill2011bayesian}}  & 747 & 24/1/1 & Birth weight \\
    \textbf{Twins~\cite{twins, louizos2017causal}}  & 12042 & 50/1/1 & The birth weight of the lighter twin \\
    \textbf{NSLM~\cite{yeager2019national}}  & 10391 & 10/1/1 &  Synthetic normally distributed random variable \\
    \textbf{CIFAR-10N~\cite{wei2022learning}}  & 50000 & 32x32x3/1/2 & Annotation time \& Label variability \\
    \textbf{CIFAR-10C~\cite{hendrycks2019robustness}}  & 40000 & 32x32x3/1/2 & Corruption severity \& type \\

    \bottomrule[1.1pt]
    \end{tabular}%
}
 
    \label{tab:real_datasets_info}

\end{table}

\subsection{CIFAR-10N dataset details}\label{sec:cifar10}

CIFAR-10N~\cite{wei2022learning} is a variation of the CIFAR-10~\cite{krizhevsky2009learning} in which the labels are given by human annotators. In this task, $X_i$ is an image from the CIFAR-10 dataset. The response $Y_i\in\{1,...,10\}$ is the noisy image label chosen by the first human annotator. The ratio of noisy samples is 17.23\%. The privileged information $Z_i$ has two features: the working time of the first annotator, and the number of different labels chosen by the three human annotators.

\subsection{CIFAR-10C dataset details}\label{sec:cifar10c}

CIFAR-10C~\cite{hendrycks2019robustness} is a variation of the CIFAR-10~\cite{krizhevsky2009learning} dataset in which the images are contaminated by artificial corruptions. Here, we randomly apply one of the following corruptions for 15\% of the images: snow, defocus blur, pixelate, or fog. The corruption severity level is either 4 or 5, chosen with equal probability. A label of a corrupted image is flipped according to the severity and corruption type. The severities 4,5 are assigned the values 0.92, and 0.95, respectively, and the snow, defocus blur, pixelate, and fog corruptions are assigned with 0.95, 0.93, 0.9, 0.93, 0.94, respectively. The flip probability is the multiplication of the value assigned with the corresponding severity and corruption type. In total, 13.09\% of the labels are flipped. Importantly, in training time, both the image and the label are corrupted, while at inference time, only the image is corrupted and the performance is computed with respect to the clean label. In the dispersive noise setting, the label is uniformly flipped into a wrong one. In the contractive noise setting, if the image is affected by either snow or defocus blur corruption, then the noisy label is deterministically set to 2, or to 1 if the original label was already 2. If the image is corrupted by either pixelate, or fog, then the noisy label is deterministically set to 7, or to 6 if the original label was already 7. The privileged information contains two features: the severity level and the corruption type. Since CIFAR-10C contains only 10000 samples, we add 30000 clean samples from CIFAR-10 to our dataset. This way, 60\% of the 10000 CIFAR-10C images are corrupted, while the other 30000 CIFAR-10 samples are clean. In total, 15\% of the images are corrupted.

\subsection{IHDP dataset details}\label{sec:ihdp}

The Infant Health and Development Program (IHDP) dataset~\cite{hill2011bayesian}, is a semi-synthetic data containing in which the response variable is the future cognitive test scores. The feature vector includes information about the child such as child—birth weight, head circumference, weeks born preterm, birth order, first born indicator, neonatal health index, twin status, and more. The treatment $M$ is the specialist home visits indicator.
The privileged information is defined as the covariate in $X_i$ with the highest correlation to $Y_i$. We then remove this feature from $X_i$, so it cannot be used at inference time.
Since this dataset is semi-synthetic, every sample point includes both potential outcomes $Y_i(0)$, $Y_i(1)$ for every sample $i$. We therefore need to choose which samples are treated with $M=0$ and which are treated with $M=1$. In section~\ref{sec:general_exp_setup} we explain how $M$ is chosen to intentionally induce a distribution shift between the observed and test distributions.

\subsection{Twins dataset details}\label{sec:twins}

The Twins dataset contains information about twin births in the USA between 1989 and 1991. The raw data was provided by~\cite{twins}, and~\cite{louizos2017causal} introduced it as a new benchmark. In this dataset, the treatment $T=1$ indicates for being born the heavier twin, the covariates $X$ include features about the twins and their parents, and the outcome corresponds to the mortality of each of the twins in their first year of life. Since the records of the two twins are available, their mortality rates are considered as two potential outcomes, where the treatment is the indicator of being born heavier. We follow the protocol of~\cite{louizos2017causal} and focus only on twins with birth weight less than 2kg. The privileged information variable is set to the birth weight of the lighter twin. Since both potential outcomes are observed in the data, we must selectively hide one of them to simulate an observational study. We choose the treatment probability as explained in Section~\ref{sec:general_exp_setup}.

\subsection{NSLM dataset details}\label{sec:nslm}

The National Study of Learning Mindsets (NSLM) dataset~\cite{yeager2019national} is a semi-synthetic data that was analyzed in the 2018 Atlantic Causal Inference Conference workshop on heterogeneous treatment effects~\cite{carvalho2019assessing}.
See~\cite[Section 2]{carvalho2019assessing} for more information about this dataset.
In our experiments, we follow the protocol introduced in~\cite{carvalho2019assessing, lei2021conformal} to generate two synthetic potential outcomes and a synthetic PI variable. Specifically, we begin by scaling the data to have 0 mean and standard deviation of 1. We split the dataset into a training set and a validation set,  containing 80\% and 20\% samples from the entire data, respectively. Using these training and validation sets, we fit a neural network with one hidden layer of size 32 to predict $Y$ from $X$. The training parameters are as detailed in Section~\ref{sec:general_exp_setup}. This network function is denoted by $\hat{\mu}_0(\cdot)$. Similarly, we fit an XGBoost classifier to predict the original treatment variable $M$ given in the data from the feature vector $X$. We set the max\_depth and n\_estimators parameters to 2, 10, respectively. We then calibrate the estimated propensity score to have the same mean as the marginal treatment probability. This calibrated estimated propensity score is denoted by $\hat{e}(X_i)$. We generate a new treatment variable $M_i$, a synthetic PI variable $Z_i$, and a semi-synthetic target variable $Y_i(0)$ as follows.
\begin{equation}
\begin{split}
& Z_i \sim \mathcal{N}(0,{0.2}^2)\\
& E_i = \mathbbm{1} \{ Z_i \geq \text{Quantile}(0.9, Z) \text{ or } Z_i \leq \text{Quantile}(0.1, Z) \} \\
& M_i \sim Ber(\min(0.8, (1 + E_i)\hat{e}(X_i))) \\
& \tau_i = 0.228+0.05 \mathbbm{1} \{ X_{i,5} < 0.07)-0.05\mathbbm{1} \{X_{i,6} < -0.69\}-0.08\mathbbm{1} \{X_{i,1} \in \{1, 13, 14\}\} \\
& Y_i(0) = \hat{\mu}_0(X_i) + \tau_i +  (1 + E_i)Z_i.
\end{split}
\end{equation}

\subsection{Synthetic dataset details}\label{sec:syn_data}

In this section, we present the synthetic dataset used in the ablation study of the parameter $\beta$ in Appendix~\ref{sec:beta_ablation}.

The feature vectors are uniformly sampled as follows:
\begin{equation}
X_i \sim \text{Uni}(1,5)^{10},
\end{equation}
where $\text{Uni}(a,b)$ is a unifrom distribution in the range $(a,b)$. 
The PI is sampled as:
\begin{equation}
\begin{split}
E^1_i &\sim \mathcal{N}(0,1),\\
E^2_i &\sim \text{Uni}(-1,1),\\
E^3_i &\sim \mathcal{N}(0,1),\\
P_i &\sim \text{Pois}(\text{cos}(E_i^2 +0.1 )) *E^2_i , \\
Z_i &\sim P_i + 2E^3_i.
\end{split}
\end{equation}
Above, $\text{Pois}(\lambda)$ is a poisson distribution with parameter $\lambda$, and $\mathcal{N}(\mu,\sigma^2)$ is a normal distribution with mean $\mu$ and variance $\sigma^2$.
Finally, the label is defined as:
\begin{equation}
\begin{split}
\beta &\sim \text{Uni}(0,1) ^ 5\\
\beta &= \beta  / || \beta ||_1 \\
U_i &= \mathbbm{1}_{Z_i < -3} + 2*\mathbbm{1}_{-3\leq Z_i \leq 1} + 8*\mathbbm{1}_{ Z_i > 1} \\
E_i &\sim \mathcal{N}(0,1) \\
Y_i &= 0.3 X_i \beta + 0.8 Z_i + 0.2 + U_i E_i .
\end{split}
\end{equation}

\section{Experimental setup}\label{sec:experimental_setup}

\subsection{General setup}\label{sec:general_exp_setup}
In all experiments, except for scarce data experiments, we split the data into a training set (50\%), calibration (20\%), validation set (10\%) used for early stopping, and a test set (20\%) to evaluate performance. See Section~\ref{sec:scarce_data_exp_setup} for the specific details in the scrace data experiments. Then, we normalize the feature vectors and response variables to have a zero mean and unit variance. 
In experiments involving missing variables, we impute them with a linear model fitted on variables that are always observed from $X, Y, Z$. The linear model is trained on samples from the training and validation sets.
For datasets that are not originally corrupted, the corruption probability is determined as follows. First, for IHDP and Twins datasets, we fit a linear model on the entire data to predict $Y$ given $X,Z$, and use its predictions as an initial value. For other datasets, we take $Z$ as the initial value. We take the initial values, set the maximal value as the 85\% quantile, and divide by the 90\% quantile of the initial values. Then, we zero the lowest 75\% values. Then, we raise all values to the exponent that achieves an average of 0.20. The result is the corruption probability. Therefore, by definition, the average corruption probability is 20\%.
In all experiments, we fit a base learning model and wrap its output with a calibration scheme. 
In regression tasks, the model is trained to learn the 5\% and 95\% conditional quantiles of $Y\mid X$.
In Table~\ref{tab:models_used} we summarize the model we used for each dataset for both tasks.
For neural network models, we used an Adam optimizer~\cite{adam} with 1e-4 learning rate, and batch size of 128. The network is composed of hidden layers of sizes: 32, 64, 64, 32, 0.1 dropout, and leaky relu as an activation function. For xgboost and random forest models, we used 100 estimators. We train the networks for 1000 epochs, but stop the training earlier if the validation loss does not improve for 200 epochs, and in this case, the model with the lowest validation loss is chosen. In our experiments, we use the xgboost package~\cite{xgboost} and the scikit-learn package~\cite{scikit-learn} from random forest. The neural networks were implemented with the pytorch package~\cite{pytorch}.
Regarding the hyper-parameters of the calibration schemes, in all experiments, we set the parameter $\beta$ of \ttmethod to $\beta=0.005$ and the parameter $\beta$ of \ttbaseline to $\beta=0.05$.

\begin{table}[htbp]
\caption{The learning models used for each dataset.}
\setstretch{1.5}
  \centering
\scalebox{0.8}{
\centering
\begin{tabular}{ccc}
    \toprule[1.1pt]
    \textbf{Dataset Name} & \textbf{Base learning model} & \textbf{Corruption probability estimator} \\
    \midrule
    \textbf{Facebook1~\cite{facebook_data}}  & Neural network & Neural network \\
    \textbf{Facebook2~\cite{facebook_data}}  & Neural network & Neural network \\
    
    \textbf{Bio~\cite{bio_data}}  &  Neural network & Neural network \\
    
    \textbf{House~\cite{house_data}}  &  Neural network & Neural network \\

    \textbf{Meps19~\cite{meps19_data}}  &  Neural network & Neural network\\
    \textbf{Blog~\cite{blog_data}}  &  Neural network & Neural network \\
    
    \textbf{IHDP~\cite{hill2011bayesian}}  &  XGBoost & XGBoost \\
    \textbf{Twins~\cite{twins, louizos2017causal}}  &  XGBoost & XGBoost \\
    \textbf{NSLM~\cite{yeager2019national}}  &  XGBoost & XGBoost \\
    \textbf{CIFAR-10N~\cite{wei2022learning}}  &  Resnet-18 & Resnet-18 when using $x$ or Random forest when using only $z$\\
    \textbf{CIFAR-10C~\cite{hendrycks2019robustness}}  &  Resnet-18 & Resnet-18 when using $x$ or Random forest when using only $z$ \\

    \bottomrule[1.1pt]
    \end{tabular}%
}
 
    \label{tab:models_used}

\end{table}

\subsection{Experimental setup for scarce data experiment}\label{sec:scarce_data_exp_setup}

In this section, we detail the experimental setup employed in the experiment in Section~\ref{sec:causal_inference_exp}. In this experiment, we split the data into a training set (30\%), a validation set (10\%), and a test set (60\%). Furthermore, the nominal coverage level was set to $1-2\alpha=90\%$. For each training sample $i$, we fit an XGBoost quantile regression model using the entire training set, except for the $i$-th sample, to obtain a leave-one-out model. Then, we calibrated the model outputs with each calibration scheme, using the leave-one-out models.

\subsection{Experimental setup for CIFAR-10N and CIFAR-10C experiments}
We use a ResNet-18 network~\cite{he2016deep} as a base model. If we used both $X$ and $Z$ as an input to the model, we forward $X$ through the CNN, and then through a linear layer with an output dimension of 16. Then, we concatenate this result with $Z$ and forward it through a network with hidden layers with sizes 32, 64, 64, 32. When fitting the CNN, we train the model with batches of size 32 for 50 epochs, and choose the model with the lowest validation loss. Additionally, we apply a random augment transform to improve performance.
For the \texttt{Two-Staged} and \texttt{Naive WCP} calibration schemes, we use only $X$ to estimate the corruption probability, and for the infeasible \texttt{WCP} and \texttt{PCP} we use only $Z$ to estimate it.

\subsection{Machine’s spec}\label{sec:machine_spec}

The resources used for the experiments are:
\begin{itemize}
    \item \textbf{CPU}: Intel(R) Xeon(R) E5-2650 v4.
    \item \textbf{GPU}: Nvidia titanx, 1080ti, 2080ti.
    \item \textbf{OS}: Ubuntu 18.04.
\end{itemize}

\subsection{Computational resources}\label{sec:compute_resources}
The computation efficiency of the proposed algorithm is dominated by the efficiency of the base learning model. The reason is that our calibration scheme only requires one pass over all calibration samples, as explained in Section~\ref{sec:efficient_pcp}.

\newcommand{\smallerplotwidth}{0.745}
\newcommand{\tabularsmallerplotwidth}{0.77}

\section{Additional experiments}\label{sec:additional_experiments}
In this section, we provide additional experiments and supply additional results. In all experiments, whenever possible, we display the performance of an uncalibrated model, \texttt{Naive CP} which uses clean and noisy samples, \texttt{Naive CP} which uses only the clean samples, \texttt{Weighted CP} with the following weights: (i) estimated using only from $X$, (ii) estimated using $Z$, and (iii) oracle weights as a function of $Z$. We further employ the \texttt{Two-Staged} algorithm and \texttt{PCP} with these three options for the weight function. It is important to note that it is not applicable in practice to apply \texttt{Weighted CP} or \texttt{Two-Staged CP} with weights estimated from $Z$, since they require the test privileged information $Z^\text{test}$. We conduct these experiments for demonstration purposes. Nevertheless, \texttt{PCP} is applicable with any of these weights, as it does not use $Z^\text{test}$.

\subsection{Causal inference tasks experiments}\label{sec:additional_causal_inference_exp}
We follow the protocol in Section~\ref{sec:causal_inference_exp}, when our goal is to estimate the uncertainty of unknown response under no treatment $Y_{n+1}(0)$. As explained in Section~\ref{sec:causal_inference_exp}, valid prediction intervals for $Y_{n+1}(0), Y_{n+1}(1)$ can be combined to construct a reliable interval for the individual treatment effect (ITE), which is valuable in many applications~\cite{brand2010benefits, morgan2001counterfactuals, xie2012estimating, florens2008identification}.

We begin with the semi-synthetic IHDP~\cite{hill2011bayesian} dataset, in which the objective is to analyze the effect of specialist home visits on future cognitive test scores. See Section~\ref{sec:ihdp} for more information about this dataset.
Figure~\ref{fig:full_ihdp} displays the coverage rate and interval lengths of the prediction intervals constructed by the calibration schemes. This figure shows that the naive techniques: \texttt{Uncalibrated}, naive jackknife+ (\texttt{Naive CP}), and naive \texttt{JAW}, which uses weights estimated only from $X$, do not achieve the desired coverage level. This is not a surprise, as the naive approaches do not hold statistical guarantees. In contrast, \texttt{JAW} which uses $Z^\text{test}$, and the proposed \ttmethod construct uncertainty intervals that achieve the desired coverage level. This is also anticipated since these methods are supported by theoretical guarantees. Nevertheless, the versions of \texttt{JAW} that achieve a valid coverage rate are infeasible in practice, as they require $Z^\text{test}$, which is unknown in our setting. In conclusion, Figure~\ref{fig:full_ihdp} suggests that \ttmethod is the only applicable calibration scheme that achieves a valid coverage level in this experiment.

Next, we turn to the Twins dataset~\cite{twins} which contains records about newborn twin babies, and the response variable is the mortality indicator. In Section~\ref{sec:twins} we provide additional information about this dataset. Figure~\ref{fig:full_twins} presents the performance of each calibration scheme on the Twins dataset. This figure indicates that \ttnaive tends to undercover the response variable while the two-staged baseline tends to overcover it. In contrast, observe that \ttwcp and \ttmethod achieve the desired $90\%$ coverage rate when employed with oracle corruption probabilities, or with probabilities estimated from $Z$. However, we remark that these versions of \ttwcp cannot be applied in practice, since they require $Z^\text{test}$, which is unavailable in our setup. The only practical version of \ttwcp is with corruption probabilities estimated from $X$, which does not achieve the nominal coverage level. Notice, however, that the proposed \ttmethod can be applied with all versions of corruption probabilities, as it does not require access to $Z^\text{test}$. To conclude, \ttmethod and \ttbaseline are the only calibration schemes that are both applicable in practice and guaranteed to generate uncertainty sets with a valid coverage rate. 
In addition, Figure~\ref{fig:full_twins} reveals that \ttmethod achieves a comparable set size to the infeasible \ttwcp, indicating that \ttmethod does not lose much statistical efficiency by not using the test PI $Z^\text{test}$.

Lastly, we consider the semi-synthetic National Study of Learning Mindsets (NSLM) dataset~\cite{yeager2019national}, which examines behavioral interventions. See \cite[Section 2]{carvalho2019assessing} for information on the dataset, and Appendix~\ref{sec:nslm} for our adaptation for this dataset. The performance of each calibration scheme is provided in Figure~\ref{fig:full_nslm}. This figure shows the same trend: the naive methods undercover the response variable while the proposed \ttmethod constructs valid uncertainty sets.

\begin{figure}[h]
         \includegraphics[width=0.98\textwidth]{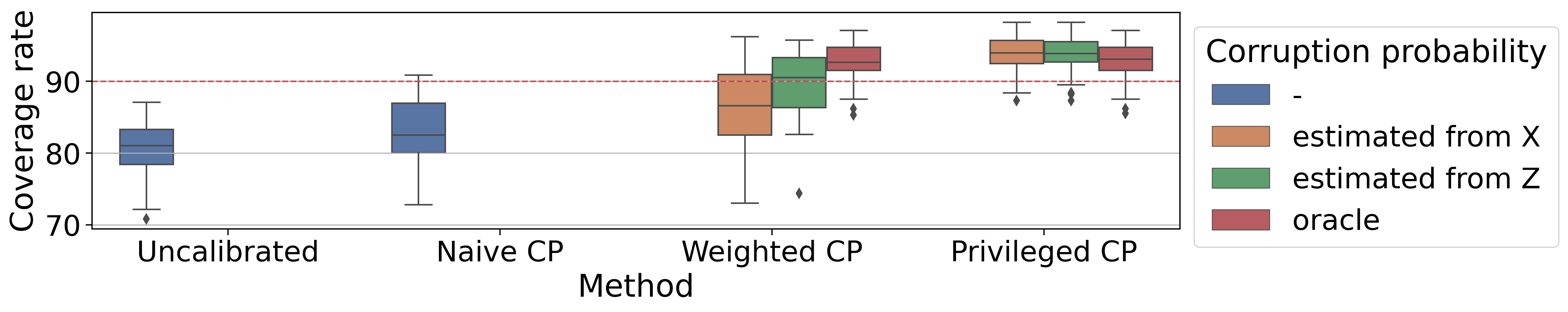}\\
       \includegraphics[width=\smallerplotwidth\textwidth]{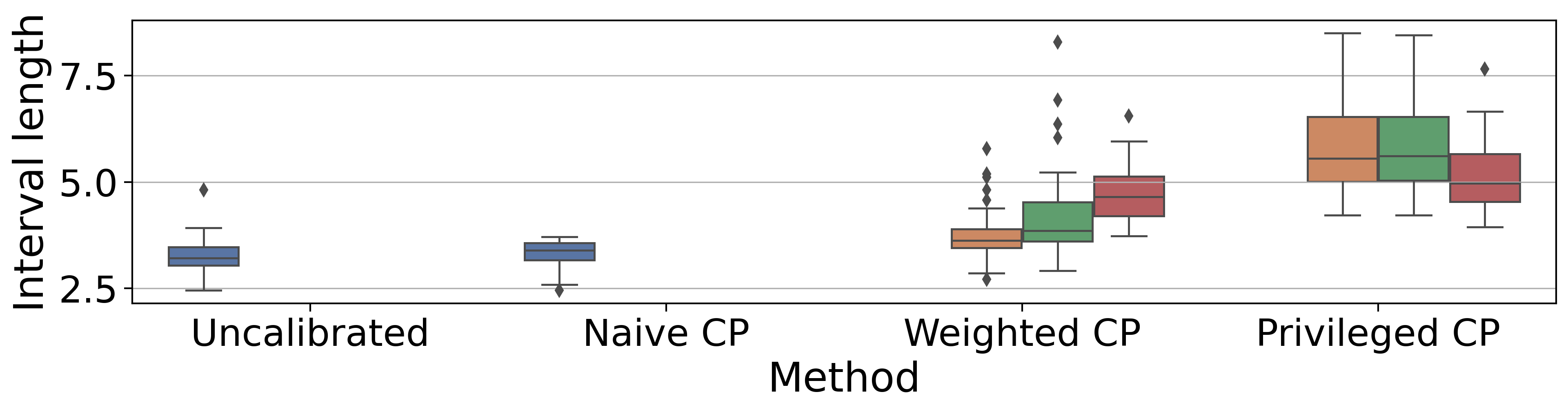}
     \caption{\textbf{IHDP dataset experiment.} The coverage rate and average interval length achieved by an uncalibrated quantile regression (\texttt{Uncalibrated}), a naive jackknife+ (\texttt{Naive CP}), \texttt{JAW} (\texttt{Weighted CP}) which estimates the corruption probability from either $X$ (orange), $Z$ (green), or uses the oracle probabilities (red), and the proposed method (\texttt{Privileged CP}) with the three options for the corruption probabilities. All methods are applied to attain a coverage rate at level $1 - 2\alpha = 90\%$. The metrics are evaluated over 50 random data splits.}
\label{fig:full_ihdp}%
\end{figure}%

\begin{figure}[h]
         \includegraphics[width=0.98\textwidth]{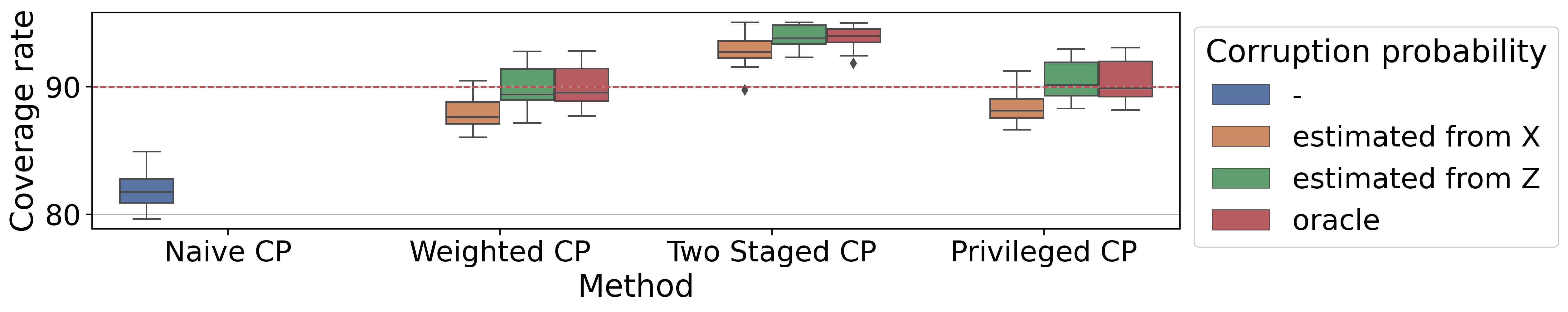}\\
         \includegraphics[width=\smallerplotwidth\textwidth]{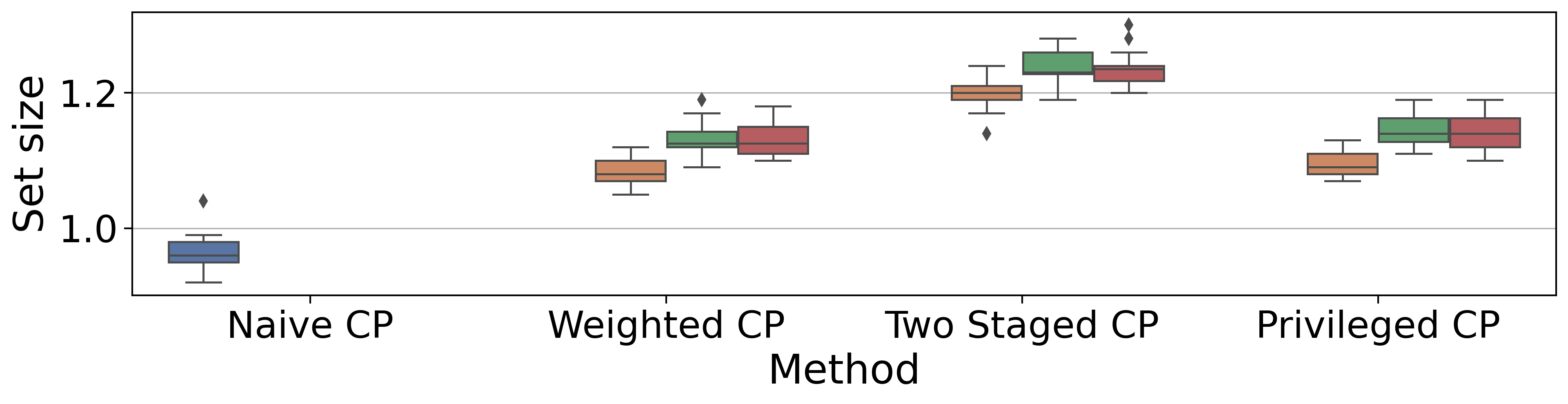}
     \caption{\textbf{Twins dataset experiment.} The coverage rate and average set size achieved by naive conformal prediction (\texttt{Naive CP}), \texttt{Weighted CP} which estimates the corruption probability from either $X$ (orange), $Z$ (green), or uses the oracle probabilities (red), the baseline \texttt{Two Staged CP}, and the proposed method (\texttt{Privileged CP}) with the three options for the corruption probabilities. All methods are applied to attain a coverage rate at level $1 - \alpha = 90\%$. The metrics are evaluated over 20 random data splits.}
\label{fig:full_twins}%
\end{figure}%

\begin{figure}[h]
         \includegraphics[width=0.98\textwidth]{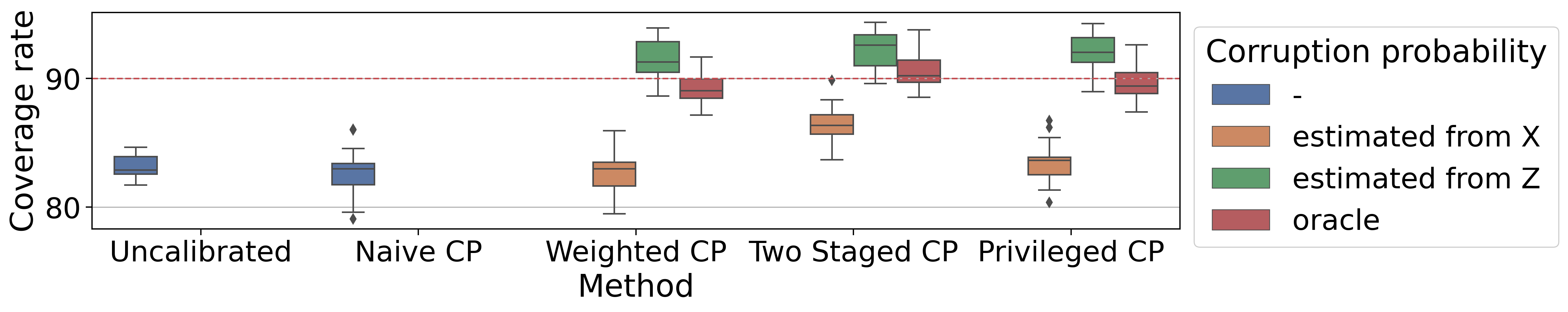}\\
         \includegraphics[width=\smallerplotwidth\textwidth]{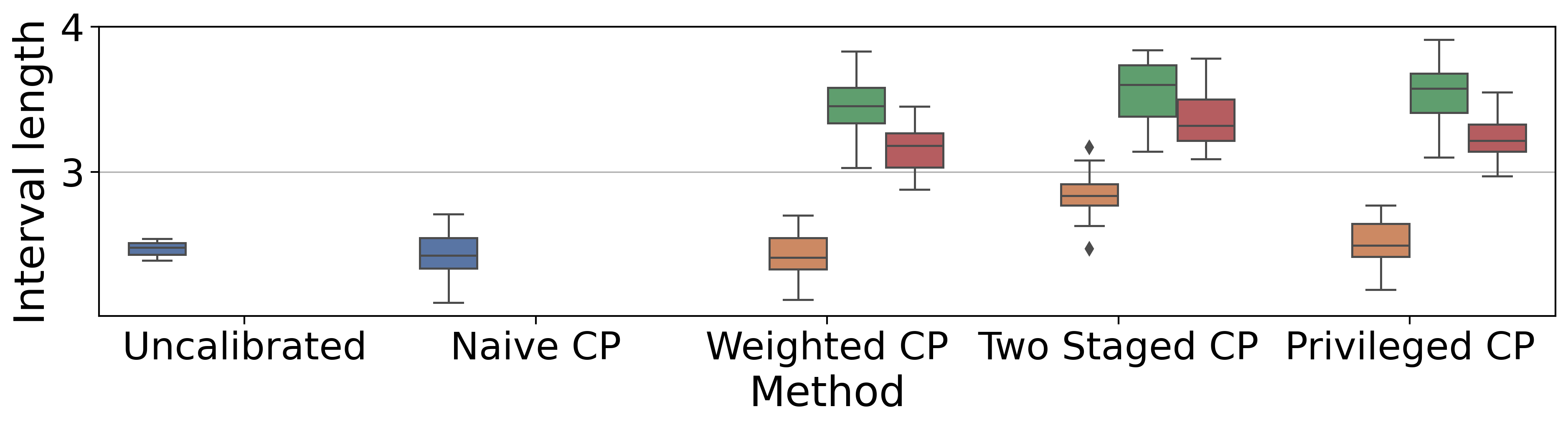}
     \caption{\textbf{NSLM dataset experiment.} The coverage rate and average interval length achieved by an uncalibrated quantile regression (\texttt{Uncalibrated}), naive conformal prediction (\texttt{Naive CP}), \texttt{Weighted CP} which estimates the corruption probability from either $X$ (orange), $Z$ (green), or uses the oracle probabilities (red), the baseline \texttt{Two Staged CP} and the proposed method (\texttt{Privileged CP}) with the three options for the corruption probabilities. All methods are applied to attain a coverage rate at level $1 - \alpha = 90\%$. The metrics are evaluated over 20 random data splits.}
\label{fig:full_nslm}%
\end{figure}%

\subsection{Noisy response variable: CIFAR-10N dataset}
In this section, we provide additional results from the CIFAR-10N experiment conducted in Section~\ref{sec:cifar10_exp}. In Figure~\ref{fig:full_cifar10} we display the performance of naive \texttt{CP}, and \ttwcp, \ttbaseline, \ttmethod, applied either with corruption probabilities estimated from $X$ or from $Z$. This figure shows that \ttwcp and \ttmethod do not achieve the desired coverage rate when the corruption probabilities are estimated from $X$. In contrast, \ttwcp and \ttmethod attain a valid coverage rate when the corruption probabilities are estimated from $Z$. This is not a surprise, as it is guaranteed by~\cite[Theorem 1]{tibshirani2019conformal} and by Theorem~\ref{thm:validity}. Lastly, this figure shows that \ttbaseline constructs uncertainty sets that are too conservative. This is also anticipated since the prediction sets of \ttbaseline encapsulate the uncertainty in both $Z$ and $Y$. Importantly, we note that \ttwcp cannot be used with corruption probabilities estimated from $Z$ (\texttt{Weighted CP} in color green in Figure~\ref{fig:full_cifar10}) since it requires access to $Z^\text{test}$, which is unknown in our setup.
\begin{figure}[h]
         \includegraphics[width=0.98\textwidth]{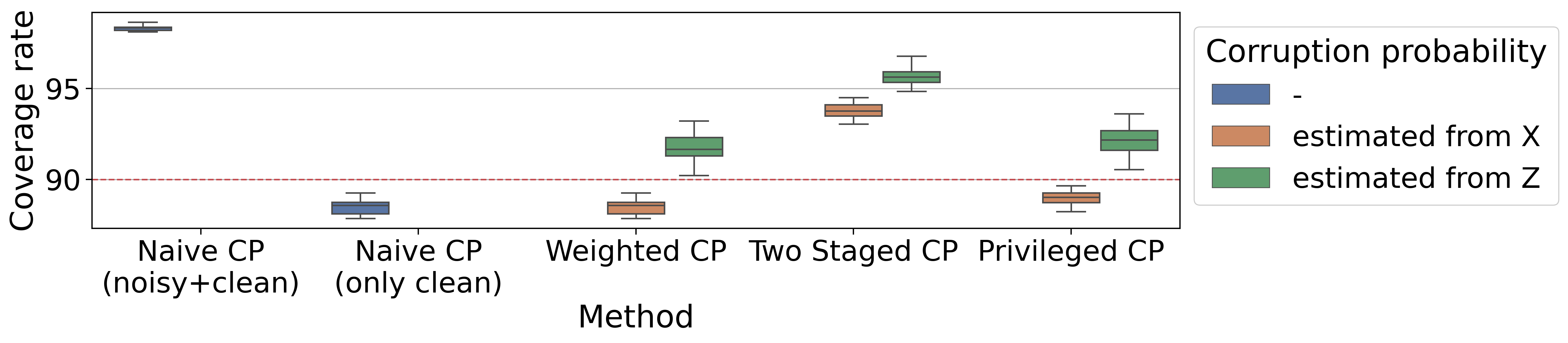}\\
         \includegraphics[width=\smallerplotwidth\textwidth]{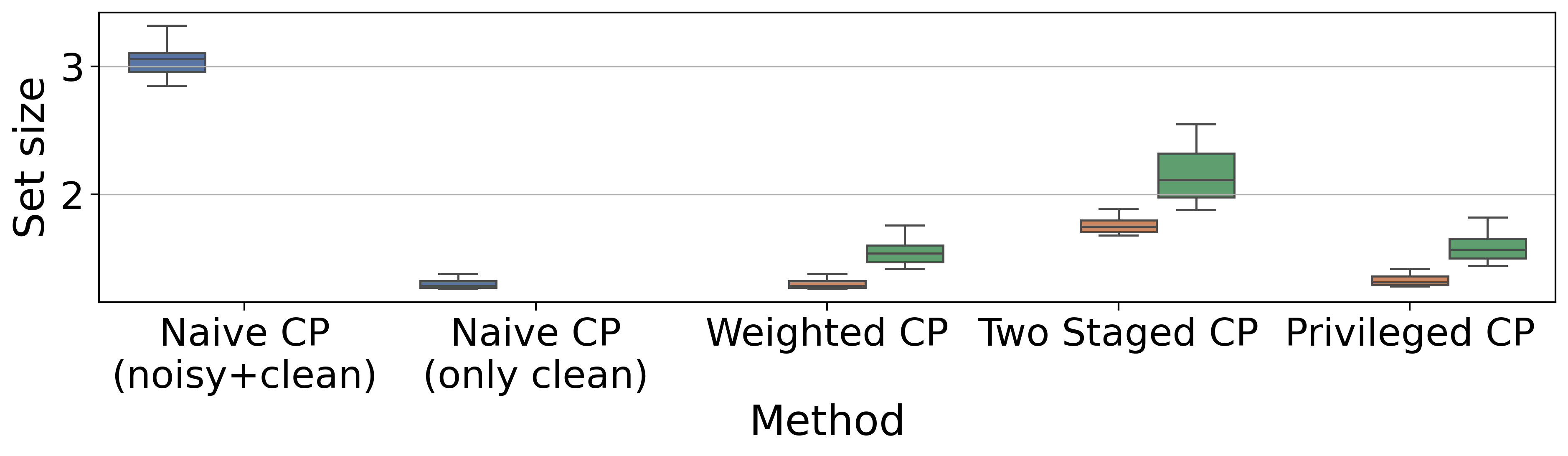}
     \caption{\textbf{Noisy response experiment: CIFAR-10N dataset.}   The coverage rate and average set size length achieved by, naive conformal prediction (\texttt{Naive CP}), using either all calibration samples (noisy + clean) or only the uncorrupted ones (only clean), \texttt{Weighted CP} which estimates the corruption probability from either $X$ (orange), $Z$ (green), the baseline \texttt{Two Staged CP} and the proposed method (\texttt{Privileged CP}) with the two options for the corruption probabilities. All methods are applied to attain a coverage rate at level $1 - \alpha = 90\%$. The metrics are evaluated over 20 random data splits.}
\label{fig:full_cifar10}%
\end{figure}%

\subsection{Noisy features and responses: CIFAR-10C dataset}\label{sec:cifar10c_exp}

We demonstrate our proposed method on the CIFAR-10C~\cite{hendrycks2019robustness} dataset, which contains pairs of a noisy image with a noisy label. In short, some of the images are corrupted, e.g., with a defocus blur, and the responses of corrupted images are artificially corrupted. Nevertheless, we note that the test images are corrupted as well, in the sense that $X(0)=X(1)$ for all samples. In this experiment, we examine two options: dispersive noise, which randomly flips the label into a different one with uniform probability, and an adversarial noise, which deterministically flips the label into a different one. In Section~\ref{sec:cifar10c} we provide the full details about this dataset. Lastly, in the following experiments, we set the desired coverage rate to 80\% since the model outputs extremely uncertain predictions for 10\% of the samples.  

Figure~\ref{fig:full_cifar10c} displays the coverage rates and set sizes achieved by each calibration scheme, on the CIFAR-10C dataset corrupted with dispersive noise. This figure indicates that by considering the noisy labels, \texttt{Naive CP} achieves a conservative coverage rate, which is consistent with the work of~\cite{label_noise, sesia2023adaptive}. Nevertheless, when applied without the noisy labels, \texttt{Naive CP} tends to undercover the correct response, similarly to the two-staged baseline \ttbaseline. Observe also that the feasible version of \ttwcp, which uses weights estimated only from $X$ achieves under coverage. In striking contrast, our proposed \ttmethod covers the response at the desired coverage rate. This trend also applies for the CIFAR-10C dataset corrupted with adversarial noise, as presented in Figure~\ref{fig:full_cifar10c_adversarial}. This figure illustrates that all feasible baselines suffer from undercoverage, except for the two-staged baseline which achieves an over-conservative coverage rate. In contrast, our proposal achieves the nominal coverage level.

\begin{figure}[h]
         \includegraphics[width=0.98\textwidth]{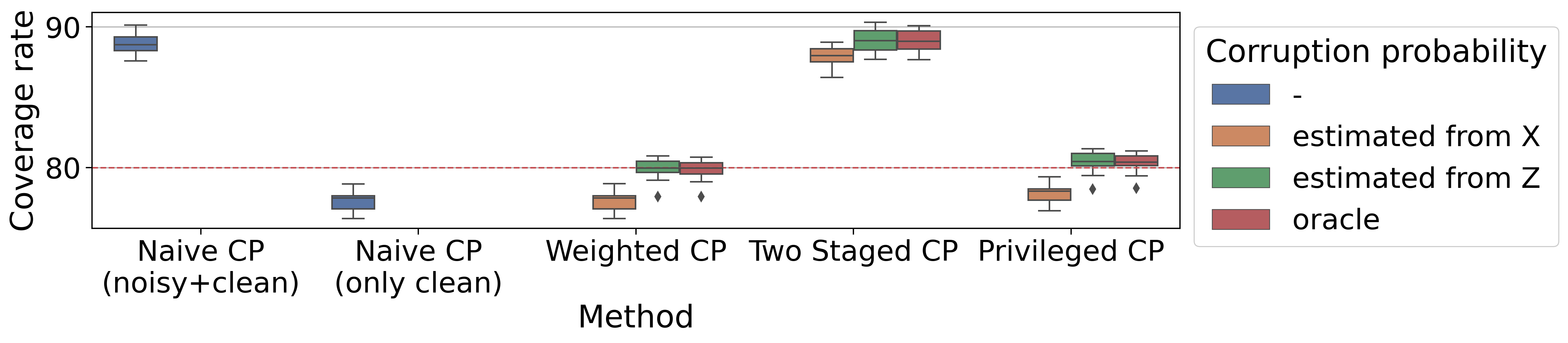}\\
         \includegraphics[width=\smallerplotwidth\textwidth]{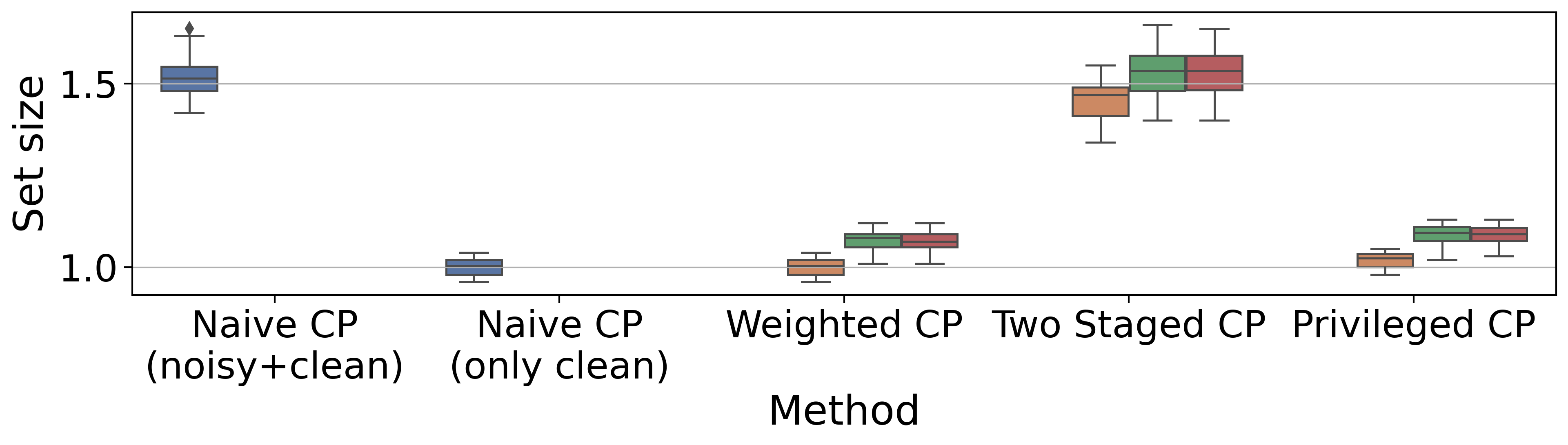}
     \caption{\textbf{Dispersive label-noise experiment: CIFAR-10C dataset.} The coverage rate and average set size length achieved by, naive conformal prediction (\texttt{Naive CP}), using either all calibration samples (noisy + clean) or only the uncorrupted ones (only clean), \texttt{Weighted CP} which estimates the corruption probability from either $X$ (orange), $Z$ (green), or uses the oracle probabilities (red), the baseline \texttt{Two Staged CP} and the proposed method (\texttt{Privileged CP}) with the three options for the corruption probabilities. All methods are applied to attain a coverage rate at level $1 - \alpha = 90\%$. The metrics are evaluated over 20 random data splits.}
\label{fig:full_cifar10c}%
\end{figure}%

\begin{figure}[h]
         \includegraphics[width=0.98\textwidth]{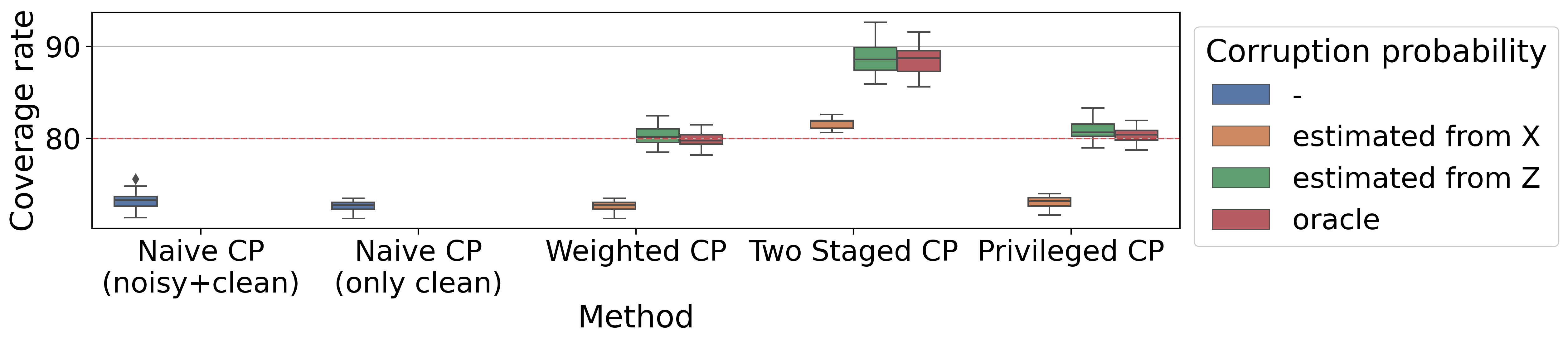}\\
         \includegraphics[width=\smallerplotwidth\textwidth]{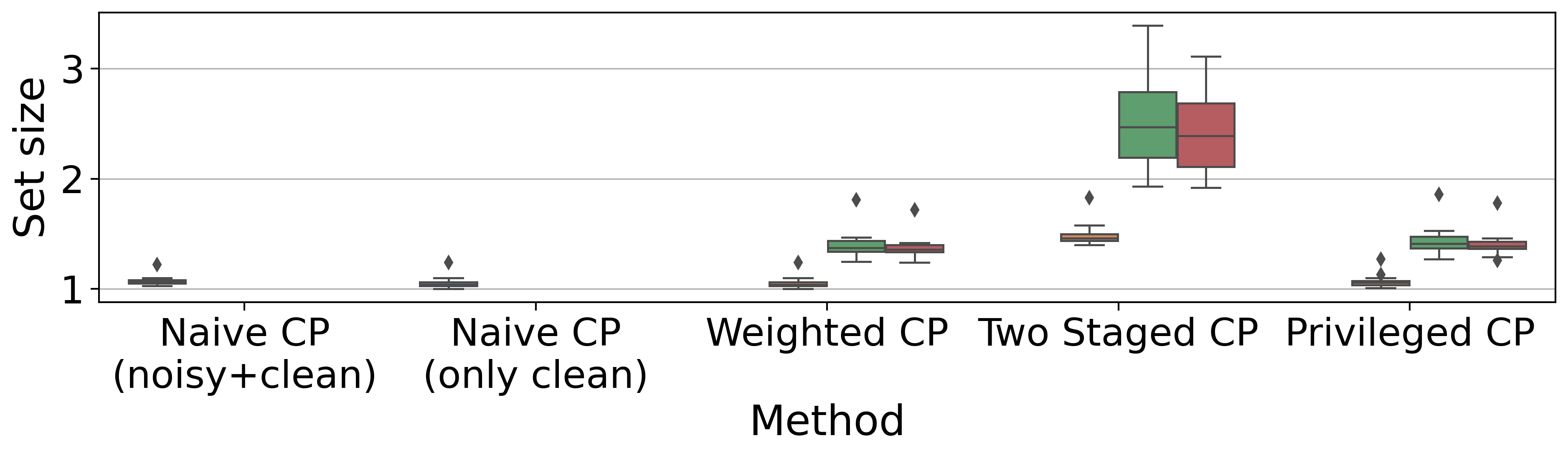}
     \caption{\textbf{Contractive label-noise experiment: CIFAR-10C dataset.} The coverage rate and average set size length achieved by, naive conformal prediction (\texttt{Naive CP}), using either all calibration samples (noisy + clean) or only the uncorrupted ones (only clean), \texttt{Weighted CP} which estimates the corruption probability from either $X$ (orange), $Z$ (green), or uses the oracle probabilities (red), the baseline \texttt{Two Staged CP} and the proposed method (\texttt{Privileged CP}) with the three options for the corruption probabilities. All methods are applied to attain a coverage rate at level $1 - \alpha = 90\%$. The metrics are evaluated over 20 random data splits.}
\label{fig:full_cifar10c_adversarial}%
\end{figure}%

\subsection{Tabular data experiments}
In this section, we conduct a series of experiments with the datasets: facebook1, facebook2, bio, house, meps19, and blog. In each experiment, we apply a different corruption to either the covariates or the response variable.
Additionally, for the \ttbaseline, \texttt{Infeasible WCP} and the proposed \ttmethod, we use the oracle corruption probabilities when computing the weights $w(z)$. However, the \texttt{Naive WCP} method uses corruption probabilities estimated only from $x$ using a neural network classifier. Section~\ref{sec:real_dataset_details} provides information about these datasets and Section~\ref{sec:general_exp_setup} describes the experimental setup.

\subsubsection{Noisy response}\label{sec:tabular_noisy_response_exp}
We examine two artificial noise functions corrupting the response. The first noise is contractive, which reduces the variability by averaging the response with its mean: $Y(1) = \frac{1}{2}(Y(0) + \mathbb{E}[Y(0)])$. The second is dispersive, which adds to the response variable a normally distributed random noise with mean 0 and standard deviation $5$ times the standard deviation of $Y(0)$. 

We begin with the dispersive noise experiment. Figure~\ref{fig:dispersive_noised_y} shows the performance of each calibration scheme in this setup. This figure indicates that the uncalibrated model and naive \texttt{CP} construct too conservative intervals. This result is consistent with the findings of~\cite{label_noise, sesia2023adaptive} which suggest that naively applying \texttt{CP} method on data with dispersive label-noise leads to conservative uncertainty sets. Nevertheless, \texttt{Naive WCP} achieves a coverage rate that is too low, possibly because the corruption probability estimates are not sufficiently accurate. Also, the two-staged baseline tends to output conservative intervals, as anticipated. Lastly, \texttt{Infeasible WCP} and the proposed \ttmethod consistently achieve the desired coverage rate for all datasets. This is not a surprise, as a theoretical guarantee supports this result.

We now turn to the contractive noise experiment, and report in Figure~\ref{fig:contractive_noised_y} the coverage rate and interval lengths of the prediction intervals constructed by each calibration scheme. This figure shows that the uncalibrated model and naive \texttt{CP} generate intervals that undercover the correct outcome. This is anticipated, as the contractive noise confuses these techniques to `think' that the underlying uncertainty is small, leading them to produce too small prediction intervals. In addition, the baselines \texttt{Naive WCP} and \texttt{Two Staged CP} construct too wide intervals, that tend to overcover the response. In contrast, \texttt{Infeasible WCP} and the proposed \ttmethod achieve the desired coverage rate.

\begin{figure}[h]
         \includegraphics[width=0.98\textwidth]{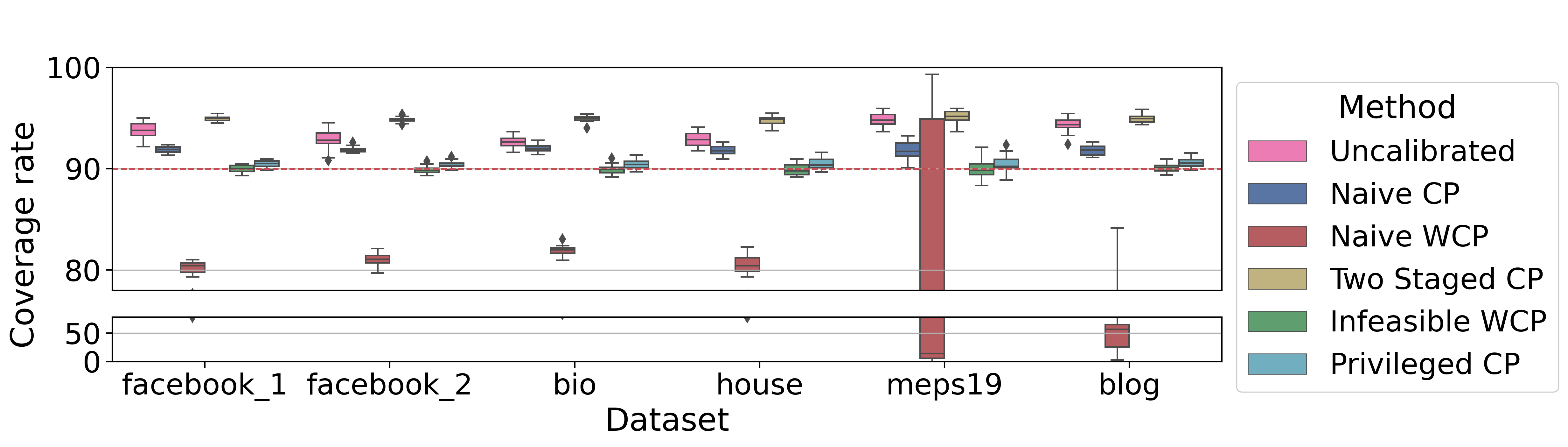}\\
         \includegraphics[width=\tabularsmallerplotwidth\textwidth]{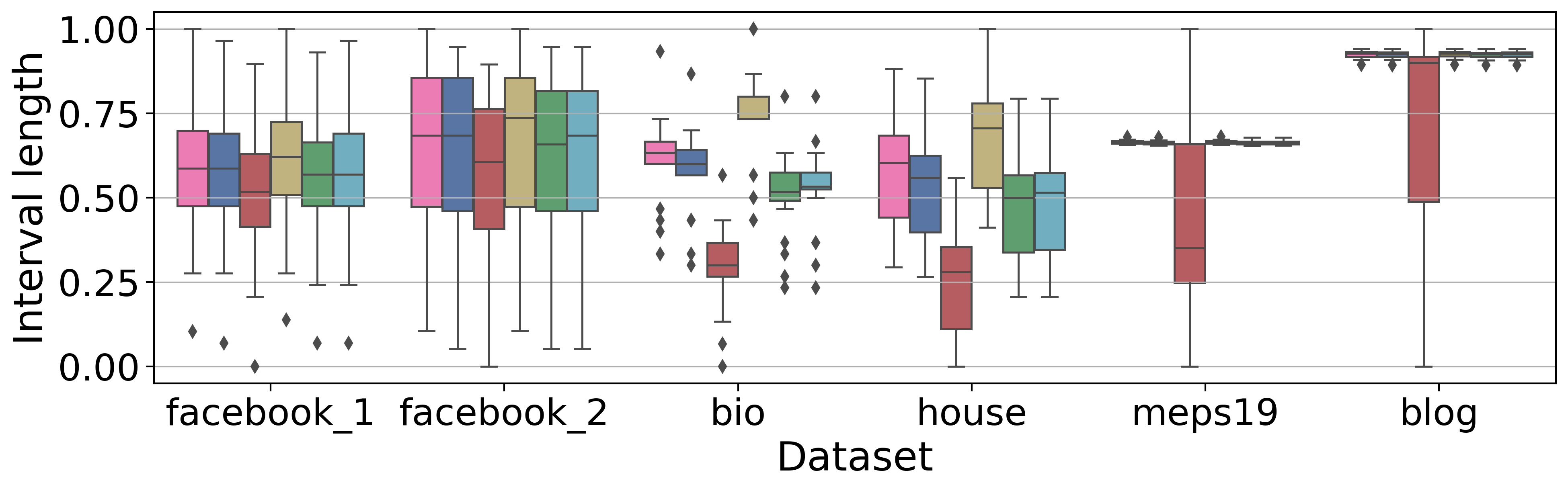}
     \caption{\textbf{Dispersive label-noise experiment: tabular datasets.} The coverage rate and average interval length achieved by an uncalibrated quantile regression (\texttt{Uncalibrated}), naive conformal prediction which uses both clean and noisy samples (\texttt{Naive CP}), \ttwcp which estimates the corruption probability from $X$ (\texttt{Naive WCP}), the baseline \texttt{Two Staged CP}, \ttwcp which uses the oracle corruption probabilities (\texttt{Infeasible WCP}), and the proposed method (\texttt{Privileged CP}) that uses to the oracle corruption probabilities. All methods are applied to attain a coverage rate at level $1 - \alpha = 90\%$. The metrics are evaluated over 20 random data splits.}
\label{fig:dispersive_noised_y}%
\end{figure}%

\begin{figure}[h]
         \includegraphics[width=0.98\textwidth]{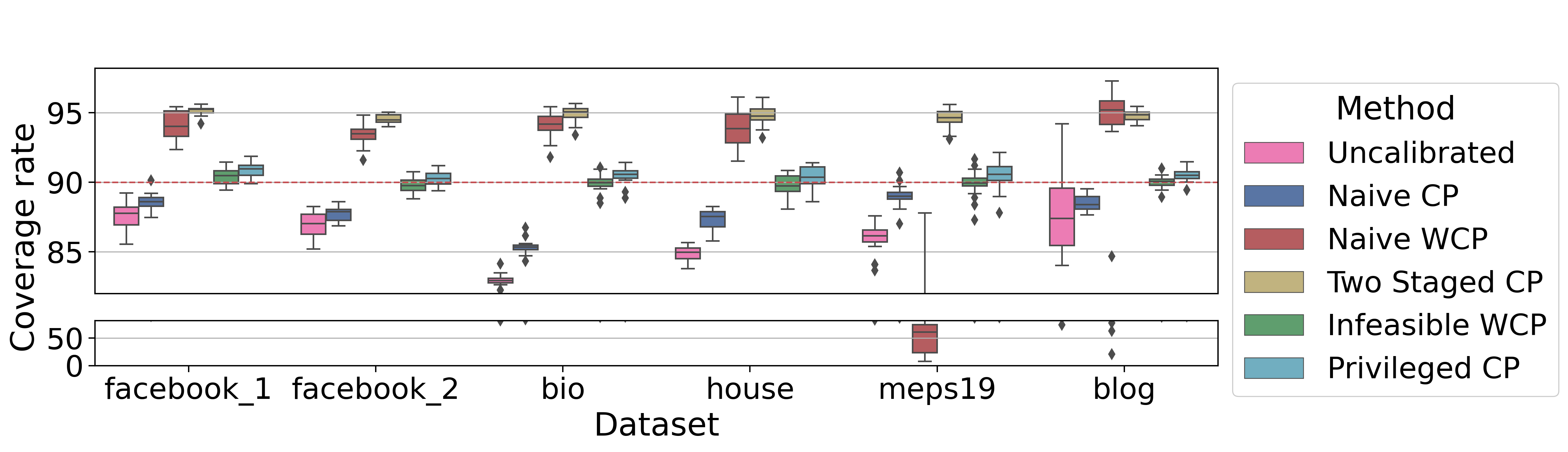}\\
         \includegraphics[width=\tabularsmallerplotwidth\textwidth]{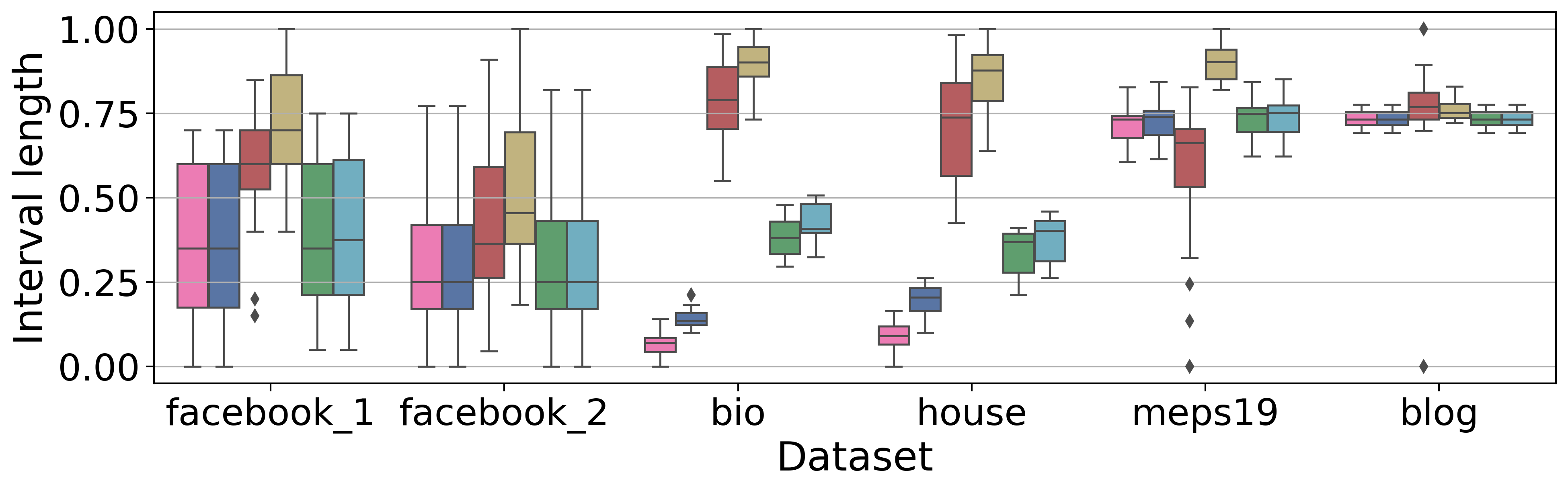}
     \caption{\textbf{Contractive label-noise experiment: tabular datasets.} The coverage rate and average interval length achieved by an uncalibrated quantile regression (\texttt{Uncalibrated}), naive conformal prediction which uses both clean and noisy samples (\texttt{Naive CP}), \ttwcp which estimates the corruption probability from $X$ (\texttt{Naive WCP}), the baseline \texttt{Two Staged CP}, \ttwcp which uses the oracle corruption probabilities (\texttt{Infeasible WCP}), and the proposed method (\texttt{Privileged CP}) that uses the oracle corruption probabilities as well. All methods are applied to attain a coverage rate at level $1 - \alpha = 90\%$. The metrics are evaluated over 20 random data splits.}
\label{fig:contractive_noised_y}%
\end{figure}%


\subsubsection{Missing features}\label{sec:missing_features_exp}

In this section, we study the setting where the entries in the corrupted feature vector $X(1)$ are missing. Specifically, we artificially delete 20\% of the features with the highest correlation to $Y_i$ from $X_i(0)$ to obtain $X_i(1)$. The corruption indicator $M_i$ is defined similarly to other experiments, as explained in Section~\ref{sec:general_exp_setup}. Figure~\ref{fig:missing_x} shows the performance of each calibration scheme. This figure indicates that the uncalibrated model and naive \texttt{CP} tend to undercover the response variable. Also, the \texttt{Naive WCP} produces intervals with large variability. This behavior probably results from inaccurate estimates of the corruption probability $P_{M\mid X}$, as its oracle counterpart, \texttt{Infeasible WCP}, which uses the true corruption probabilities, precisely achieves the nominal coverage level. Furthermore, the baseline \texttt{Two Staged CP} generates too wide intervals that tend to overcover the response. In contrast, the proposed \ttmethod consistently achieves the desired coverage rate $1-\alpha=90\%$. 

\begin{figure}[ht]
         \includegraphics[width=0.98\textwidth]{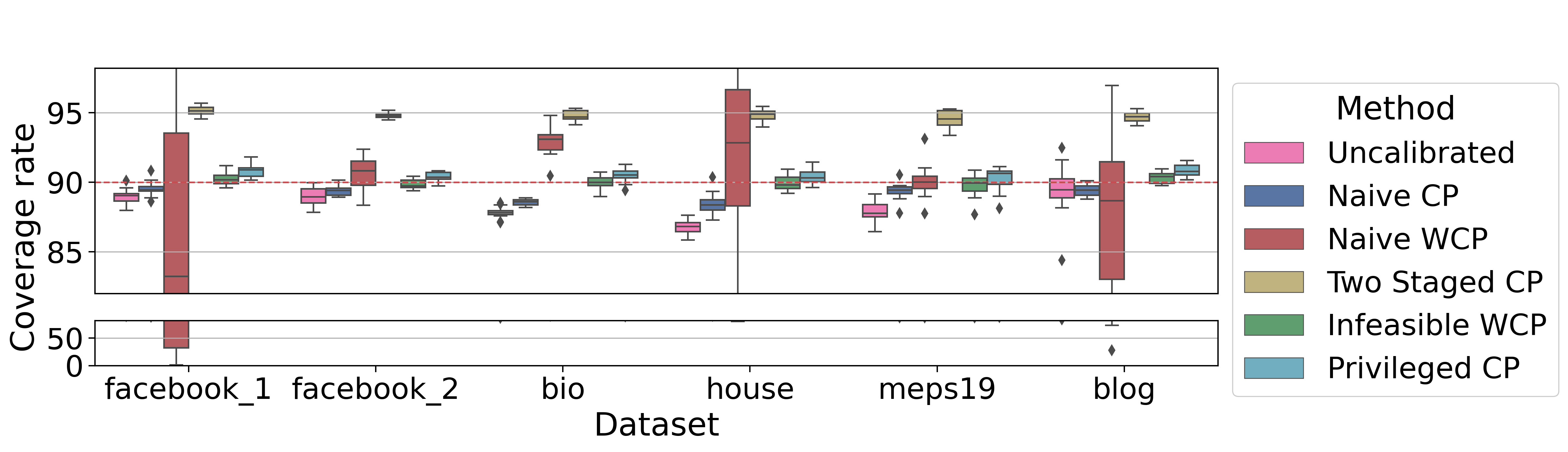}\\
         \includegraphics[width=\tabularsmallerplotwidth\textwidth]{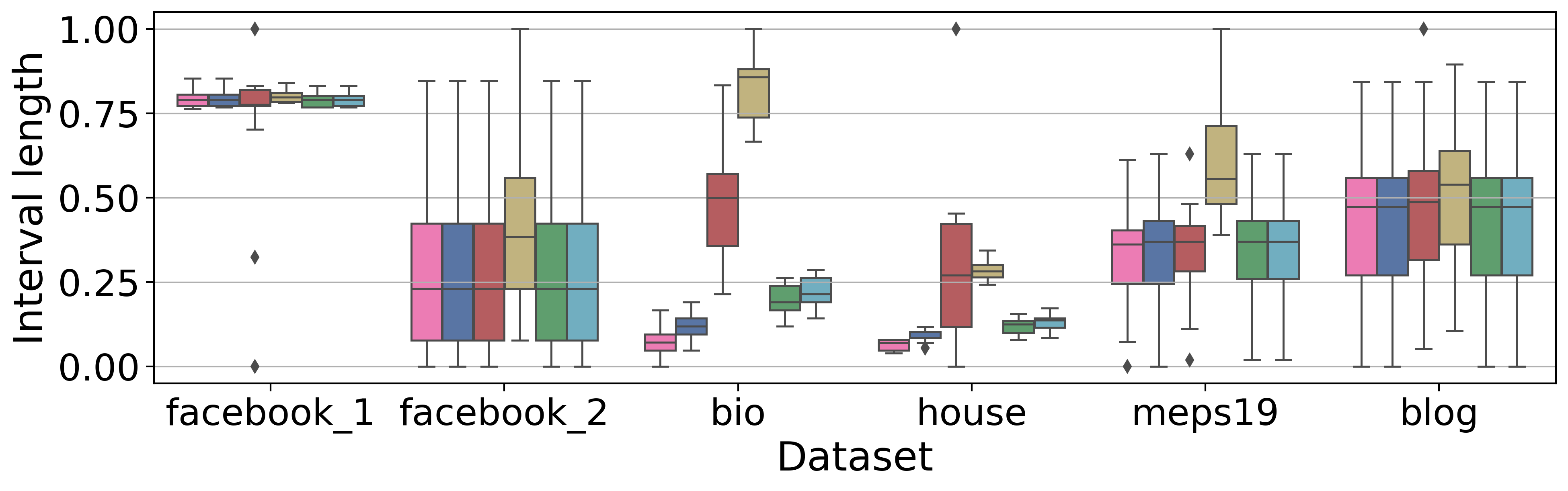}
     \caption{\textbf{Missing features experiment: tabular datasets.} The coverage rate and average interval length achieved by an uncalibrated quantile regression (\texttt{Uncalibrated}), naive conformal prediction which uses both clean and noisy samples (\texttt{Naive CP}), \ttwcp which estimates the corruption probability from $X$ (\texttt{Naive WCP}), the baseline \texttt{Two Staged CP}, \ttwcp which uses the oracle corruption probabilities (\texttt{Infeasible WCP}), and the proposed method (\texttt{Privileged CP}) that uses to the oracle corruption probabilities. All methods are applied to attain a coverage rate at level $1 - \alpha = 90\%$. The metrics are evaluated over 20 random data splits.}
\label{fig:missing_x}%
\end{figure}%



\subsection{Ablation study of the effect of \texorpdfstring{$\beta$}{b}}\label{sec:beta_ablation}
This section studies the effect of the parameter $\beta$ on the prediction sets constructed by \ttmethod. We apply \ttmethod on a synthetic data, introduced in Appendix~\ref{sec:syn_data} with different values of $\beta\in(0,\alpha)$. We follow the experimental protocol described in Appendix~\ref{sec:general_exp_setup}, and report the coverage rate and average length achieved by \ttmethod in Figure~\ref{fig:beta_ablation}. This figure indicates that the smallest intervals are achieved for $\beta$ that is close to 0, and the interval sizes are an increasing function of $\beta$. Yet, it is important to understand that different results could be obtained for different datasets. Therefore, we recommend choosing $\beta$ using a validation set, with a grid of values for $\beta$ in $(0,\alpha)$.

\begin{figure}[h]
  \centering
         \includegraphics[width=0.9\textwidth]{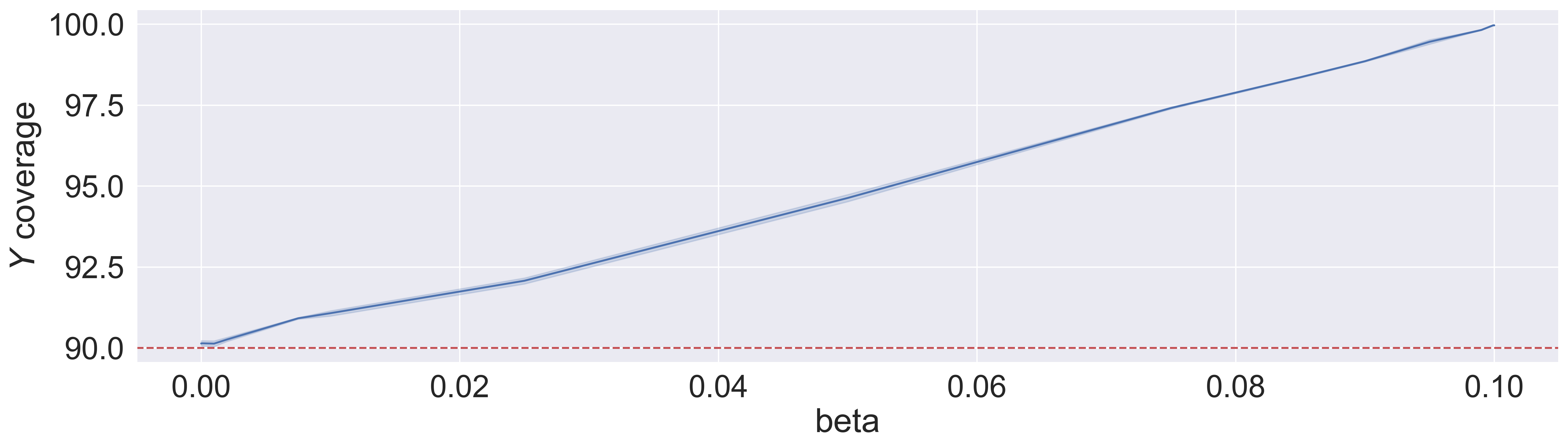}\\
         
         \includegraphics[width=0.9\textwidth]{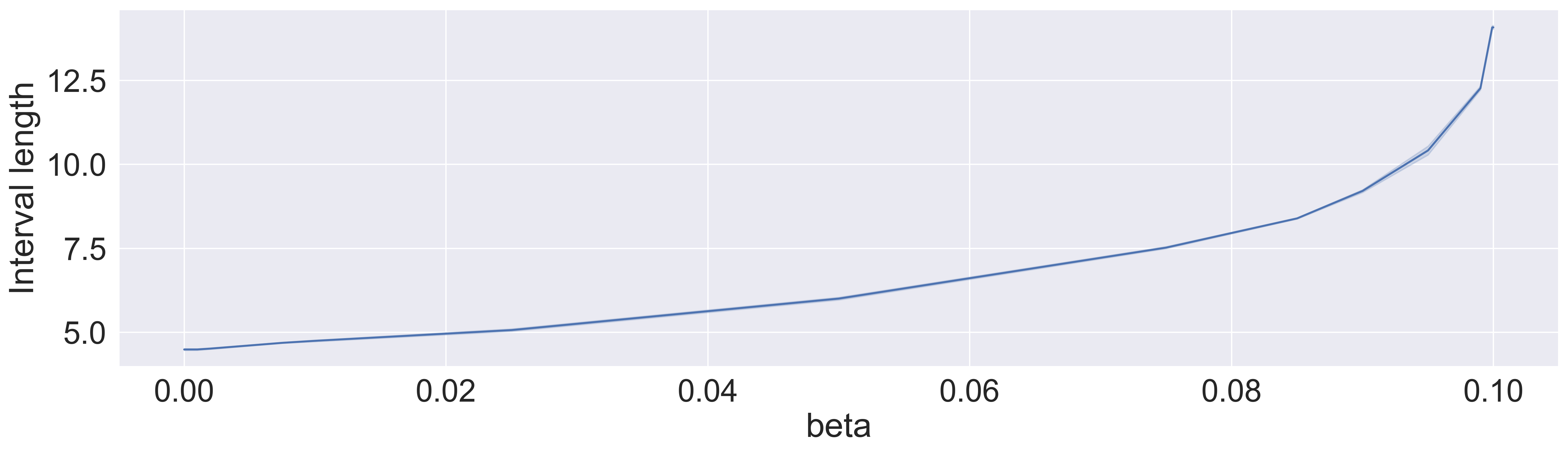}
     \caption{\textbf{Ablation study of $\beta$.} The coverage rate and average interval length achieved by \ttmethod applied to attain a coverage rate at level $1 - \alpha = 90\%$.
    The metrics are evaluated over 20 random data splits.}
\label{fig:beta_ablation}%
\end{figure}%

\end{document}